\def\year{2021}\relax
\documentclass[letterpaper]{article} % DO NOT CHANGE THIS
\usepackage{aaai21}  % DO NOT CHANGE THIS
\usepackage{times}  % DO NOT CHANGE THIS
\usepackage{helvet} % DO NOT CHANGE THIS
\usepackage{courier}  % DO NOT CHANGE THIS
\usepackage[hyphens]{url}  % DO NOT CHANGE THIS
\usepackage{graphicx} % DO NOT CHANGE THIS
\usepackage{natbib}  % DO NOT CHANGE THIS AND DO NOT ADD ANY OPTIONS TO IT
\usepackage{caption} % DO NOT CHANGE THIS AND DO NOT ADD ANY OPTIONS TO IT
\usepackage[draft]{todonotes}
\usepackage{comment}
\usepackage{amsthm}
\usepackage{subcaption}

\title{To Regularize or Not To Regularize? The Bias Variance Trade-off in Regularized AEs}
\author{
\begin{tabular}{cccc}
Arnab Kumar Mondal & Himanshu Asnani & Parag Singla & Prathosh AP \\
IIT Delhi & TIFR Mumbai & IIT Delhi & IIT Delhi \\
\small anz188380@cse.iitd.ac.in & \small himanshu.asnani@tifr.res.in & \small parags@cse.iitd.ac.in &  \small  prathoshap@ee.iitd.ac.in
\end{tabular}
}
\date{}

\usepackage{algpseudocode,algorithm,algorithmicx}
\newcommand{\printfnsymbol}[1]{%
  \textsuperscript{\@fnsymbol{#1}}%
}
\algrenewcommand\algorithmicrequire{\textbf{Precondition:}}
\algrenewcommand\algorithmicensure{\textbf{Postcondition:}}

\newtheorem{theorem}{Theorem}

\usepackage{amsmath,amsfonts,bm}

\def\eqref#1{equation~\ref{#1}}
\def\1{\bm{1}}

\def\vn{{\bm{n}}}

\def\vx{{\bm{x}}}

\def\vz{{\bm{z}}}

\DeclareMathAlphabet{\mathsfit}{\encodingdefault}{\sfdefault}{m}{sl}
\SetMathAlphabet{\mathsfit}{bold}{\encodingdefault}{\sfdefault}{bx}{n}

\newsavebox{\largestimageBD}
\newsavebox{\largestimage}

\usepackage{titling}

\begin{document}

\maketitle

\begin{abstract}
Regularized Auto-Encoders (RAEs) form a rich class of neural generative models. They effectively model the joint-distribution between the data and the latent space using an Encoder-Decoder combination, with regularization imposed in terms of a prior over the latent space. Despite their advantages, such as stability in training, the performance of AE based models has not reached the superior standards as that of the other generative models such as Generative Adversarial Networks (GANs). Motivated by this, we examine the effect of the latent prior on the generation quality of deterministic AE models in this paper. Specifically, we consider the class of RAEs with deterministic Encoder-Decoder pairs, Wasserstein Auto-Encoders (WAE), and show that having a fixed prior distribution, \textit{a priori}, oblivious to the dimensionality of the `true' latent space, will lead to the infeasibility of the optimization problem considered. Further, we show that, in the finite data regime, despite knowing the correct latent dimensionality, there exists a bias-variance trade-off with any arbitrary prior imposition. As a remedy to both the issues mentioned above, we introduce an additional state space in the form of flexibly learnable latent priors, in the optimization objective of the WAEs. We implicitly learn the distribution of the latent prior jointly with the AE training, which not only makes the learning objective feasible but also facilitates operation on different points of the bias-variance curve. We show the efficacy of our model, called FlexAE, through several experiments on multiple datasets, and demonstrate that it is the new state-of-the-art for the AE based generative models.
\end{abstract}

\section{Introduction}

\begin{figure*}[!htbp]
    \centering
    \includegraphics[trim=30 310 30 50, clip,keepaspectratio,width=\textwidth]{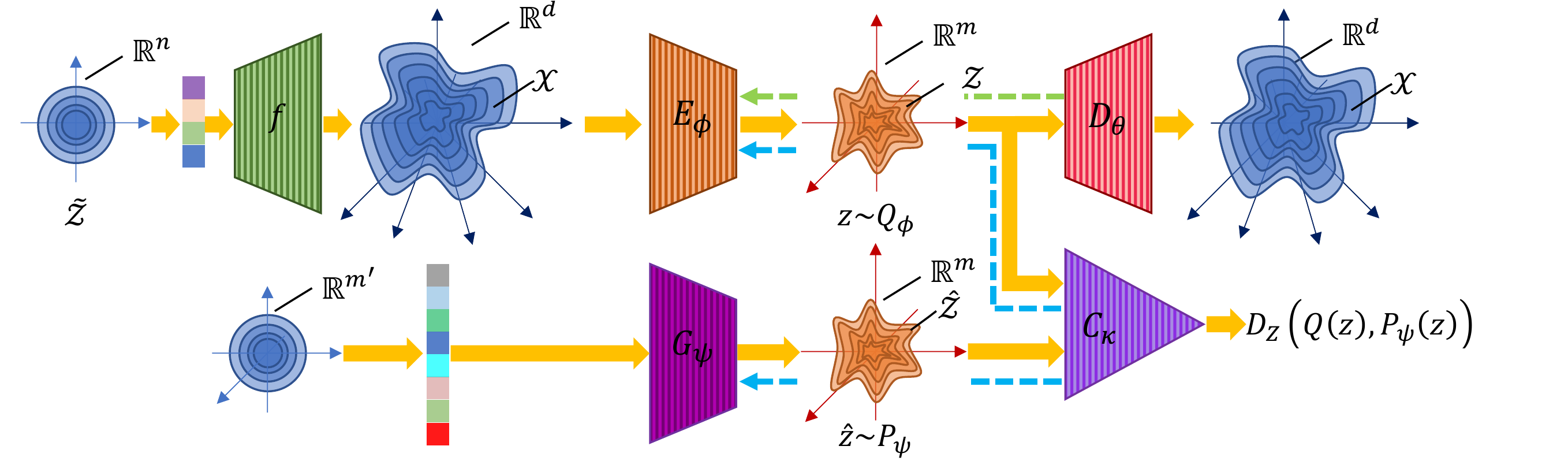}
    \caption{Our Model, FlexAE: Nature first samples an $n$-dimensional latent code from the true latent space, $\widetilde{\mathcal{Z}}$. Next, the latent code is mapped to an $n$-dimensional manifold, $\mathcal{X}$ in $\mathbb{R}^d$. The observed variables are encoded using a deterministic Encoder, $E_\phi$. The $m$-dimensional encoded representations lie in an $n$-dimensional manifold $\mathcal{Z}$. The decoder network, $D_\theta$, learns an inverse projection from the learnt latent space, $\mathcal{Z}$ to the dataspace, $\mathcal{X}$. The generator network, $G_\psi$ parameterizes the learnable prior distribution, that maps an isotropic Gaussian distribution in $\mathbb{R}^{m'}$ to any arbitrary prior $P_\psi(\vz)$ in $\mathbb{R}^m$. The critic network, $C_\kappa$ measures the distributional divergence between $Q_\phi$ and $P_\psi$. $G_\psi$ and $C_\kappa$ are jointly trained along with the Auto-Encoder. }
    \label{fig:data_gen_flexAE}
\end{figure*}
Regularized Auto-Encoder (AE) based latent variable models implicitly define a joint distribution over the input data and a lower-dimensional latent space, by approximating the true latent posterior, with a variational distribution. This variational distibution is parameterized using a neural network called the Encoder. The distribution induced by the Encoder is regularized to follow a pre-defined latent prior distribution. Subsequently, a Decoder network is trained to conditionally sample from the data distribution via optimizing a data-reconstruction metric. The parameters of the Encoder and the Decoder networks are learnt by optimizing either a  bound on the data likelihood \cite{kingma2013autoencoding} or a divergence measure between the true and generated data distributions \cite{WAE}. The framework of AE-based generative models is attractive because of its ease and stability in training, efficiency in sampling, and flexibility in architectural choices. However, despite their advantages, AE-based models have failed to reach the performance of other State-of-The-Art (SoTA) generative models \cite{dai2019diagnosing,MaskAAE}.  \par Several aspects such as the loss function used for optimization \cite{higgins2017beta,larsen2015autoencoding}, presence of conflicting terms in the optimization objective \cite{hoffman2016elbo, kim2018disentangling}, distributional choices (E.g., Gaussianity) imposed on the Encoder and Decoder \cite{zhao2019infovae,rezende2018taming}, dimensionality of the latent space used \cite{dai2019diagnosing,MaskAAE}, the mismatch between the learned and imposed prior \cite{zhao2017towards,tomczak2017vae}  have been identified as possible causes for the sub-optimal performance of the AE-based models.    
Many remedial measures, including the modification of the objective function \cite{zhao2019infovae,higgins2017beta,kim2018disentangling}, use of non-Gaussian Encoder/Decoder \cite{larsen2015autoencoding,nalisnick2016approximate}, masking of spurious latent dimensions \cite{MaskAAE}, incorporating a richer class of priors on the latent space \cite{tomczak2017vae,takahashi2019variational,klushyn2019learning}, have been proposed in the literature to address some of these issues. While these modifications have improved AE models' performance, they are still behind SoTA generative models \cite{dai2019diagnosing, MaskAAE}. In this work, we address one of these issues with the following contributions:

% It is well-recognized that the aggregated distribution learned by the Encoders of the AE-model often fails to match the assumed latent prior leading to sub optimal performances \cite{zhao2017towards,tomczak2017vae,dai2019diagnosing}. In this paper, we identify and formalize a critical issue that leads to the sub-optimality of AE-models with deterministic encoder-decoder pair (E. g., WAE/AAE):  The choice of the latent prior oblivious to the dimensionality of the true generative latent space would lead to an infeasible optimization objective. Furthermore, in the finite data regime, we demonstrate that that there exists a bias-variance trade-off that comes with imposition of any latent prior. We subsequently propose a framework that can flexibly learn the latent prior jointly with the AE-models, that addresses both the issues. 

\begin{enumerate}
    \item We theoretically establish that in a  deterministic AE based generative model, choosing a latent prior distribution supported on the entire space, leads to infeasible optimization objective, when the `true' latent space has dimensionality that is other than that of the model's latent space.
    \item We argue that even with matched dimensionality, there exists a bias-variance trade off that arises from the choice of any assumed latent prior, whenever there is a finite data. 
    \item As a remedy, we propose a new model, which we call FlexAE, that can impose flexible learnable priors on RAEs that not only make the optimization problem feasible but also facilitate a better trade off between the bias and variance on-the-go, during AE-training.
    \item We empirically demonstrate our claims through extensive experimentation on synthetic and real-world datasets by achieving significant improvement over the SoTA AE-based generative models.
\end{enumerate}

\section{Background and Related Work}
The general theme in all RAEs is to implicitly learn the joint distribution between the observed data and a latent variable, via optimizing an objective function which consists of an auto-encoding (conditional likelihood) and latent regularization term (divergence measure). Variational Autoencoder (VAE) \cite{kingma2013autoencoding} is the pioneering member of this family, in which the variational latent posterior and conditional data likelihood are respectively parameterized by probabilistic (Gaussian) Encoder and Decoder networks, while the latent prior is assumed to be an isotropic Gaussian distribution. A related class of AE-models are the Adversarial Auto-Encoders (AAEs) \cite{44904} and Wasserstein Auto-Encoders (WAEs) \cite{WAE} where a pair of deterministic Encoder-Decoder is used with Jenson-Shannon and Wasserstein distance respectively, between the aggregated latent posterior and the latent prior.\par
Even though VAEs/WAEs provide solid frameworks for AE-based generative models, several drawbacks are associated with it, which prevents them to compete with the other SoTA generative models. It is shown that there exists a conflict between the two terms of the objective, in the case of VAEs \cite{higgins2017beta,zhao2017towards,rezende2018taming}. A few remedial measures such as introduction of a tunable parameter in the objective \cite{burgess2018understanding}, use of additional penalties such as mutual information \cite{zhao2019infovae}, total correlation \cite{kim2018disentangling}, and generalised optimization objective \cite{rezende2018taming} have been proposed.
Another often discussed issue with AE-models with stochastic Encoder-Decoders is that they adopt a simple unimodal Gaussian distribution for parameterization \cite{rosca2018distribution}. To address this, \cite{nalisnick2016deep} implements a Bayesian nonparametric version of the variational autoencoder that has a latent representation with stochastic dimensionality and could represent richer class of distributions. Invertible flow-based generative models \cite{kingma2016improved, rezende2015variational} capitalize on the idea of normalizing flow for the Encoder and Decoder networks. VAE/GAN \cite{larsen2015autoencoding}, VGH/VGH++ \cite{rosca2018distribution} incorporates an adversarial learning at the Decoder so that it can represent a rich class of distributions.
\par Further, it is observed that there is a mismatch between the aggregated variational posterior and the latent prior, leading to sub-optimality of the divergence term in the objective and in turn poor generation \cite{tomczak2017vae,dai2019diagnosing}. Several methods try to alleviate this problem, broadly in two ways (i) using a richer class of parametric priors on the latent space \cite{tomczak2017vae,klushyn2019learning,kumar2020regularized} and (ii) using a post-hoc technique to minimize the divergence or sample from the latent space without regularizing it \cite{bauer2018resampled,ghosh2020from,takahashi2019variational}.  
Among the first category of methods, VampPrior \cite{tomczak2017vae} assumes the prior to be a mixture of the conditional posteriors with a set of learnable pseudo-inputs. Authors in \cite{klushyn2019learning} adapt the constrained optimization setting in \cite{rezende2018taming} and substitute the standard normal prior with a hierarchical prior and use an importance-weighted bound as the optimization objective. In \cite{huang2017learnable,kumar2020regularized,kingma2016improved}, latent priors are learned using normalizing flow based methods. Within the second category of methods, \cite{bauer2018resampled} learns to sample from a rich class of priors by multiplying a simplistic prior distribution with a learned acceptance function. In \cite{takahashi2019variational}, kernel density trick is used for matching the prior to the aggregated posterior. RAE-GMM \cite{ghosh2020from}, imposes an L$2$-norm penalty in the latent space and learns to sample from it using a Gaussian Mixture Model (GMM) on the latent space. While these methods report improvement over the SoTA metrics, not many give a theoretical justification for using richer-class of latent priors. Further, post-hoc latent samplers such as RAE-GMM \cite{ghosh2020from} do not have control over the amount of bias imposed (modulo a simple objective scaling factor), that might lead to over/under fitting as shown later. \par
However, it has been both theoretically and empirically observed that dimensionality of the latent space used has a critical impact on the performance of these models \cite{MaskAAE,dai2019diagnosing,wae_lat_dim}.  Authors in \cite{dai2019diagnosing} study the implication of the mismatch between the dimensionality of the data and the true latent space and the role of Decoder variance, in the case of AEs with stochastic Encoders. They argue a learnable variance in the Decoder would make the objective reach negative infinity even when the aggregated posterior would not match the standard Gaussian prior not because of simplistic modelling assumption but because of mismatch between data dimensionality and the true latent dimensionality. To resolve this issue they introduce a second-stage VAE, which is used on the latent space of the first stage (which is a usual VAE), where the data and the latent dimensions match. In MaskAAE \cite{MaskAAE}, the authors noted that the generation quality degrades when there is a mismatch between the dimensionality of the true and the assumed latent space of a deterministic AE. They develop a procedure to explicitly zero-out (mask) the spurious latent dimensions via a learnable masking layer. In this backdrop, however, ours is the first study to theoretically demonstrate the in-feasibility of the objective of a deterministic Generative AE model such as WAE \cite{WAE}, under a fixed prior in relation with the mismatched latent dimensionality.

% \par
% Motivated by these, in this work, we investigate the relation between the dimensionality of the assumed latent space and fixation of a distribution on the latent prior in generative AE models with deterministic Encoder-Decoder pair, such as WAE \cite{WAE}, AAE \cite{44904}.  %In perspective, we believe that ours is the first study to theoretically demonstrate the in-feasibility of the RAE objective under a prior fixation. 

\section{Proposed Method}
%In this section, we first theoretically show the infeasibility of the objective function of WAE/AAE with a mismatch between the true and the assumed latent spaces, followed by which we show that there exists a bias-variance trade-off with a choice of latent prior despite correcting the dimensionality issue. Finally we propose a model that handles both the problems effectively. 

\subsection{Optimality of the Latent Space of WAEs}
We start by assuming that the true data is generated in nature via a two-step process. First, the true latent variables are sampled from  an $n$-dimensional space, $\widetilde{\mathcal{Z}}$ according to some continuous distribution in $\mathbb{R}^n$. Next, a non-linear function, $f: \widetilde{\mathcal{Z}} \to \mathcal{X}$ maps the true latent space, $\widetilde{\mathcal{Z}}$ to the observed data space, $\mathcal{X} \subseteq \mathbb{R}^d$, with $d>>n$, in most practical cases. In other words, observed data $\vx$ lies on $\mathcal{X}$, an $n$-dimensional manifold embedded in $\mathbb{R}^d$. We make a benign assumption on $f$ that it can be represented using neural networks with sigmoidal (or hyperbolic tangent, ReLU, Leaky ReLU etc.) activations to arbitrary closeness. Under this model, the data could be seen as lying in an $n$-dimensional manifold within $\mathbb{R}^d$, with an underlying ground truth distribution $P_d(\vx)$. The objective of the WAE model is to estimate (or learn to sample from) the distribution $P_d(\vx)$, given some i.i.d. samples drawn from it.  The distribution learned by an RAE, denoted by $P_\theta(\vx)$ is given by $ P_\theta(\vx) = \int_\mathcal{Z}P_\theta(\vx|\vz) dP_z$, where $P_\theta(\vx|\vz)$ is the distribution parameterized by a deterministic Decoder neural network $D_\theta(\vz)$ and $P_Z(\vz)$ is the latent prior defined on an $m$-dimensional space, $\mathcal{Z}$. The distribution $P_\theta(\vx)$ is estimated by minimizing the Wasserstein distance between $P_\theta(\vx)$ and $P_d(\vx)$ which is obtained by solving the following optimization problem \cite{WAE}: 

\begin{equation}
\begin{gathered}
\mathop{\inf}_{\phi, \theta}\Bigg(\mathop{\mathbb{E}}_{P_d}\mathop{\mathbb{E}}_{Q_\phi(\vz|\vx)}\Big[c\big(\vx, D_\theta(\vz)\big)\Big]\Bigg)\\
\text{such that } Q_\phi(\vz) = P_Z(\vz)
\label{eqn:wae_constrained_obj}
\end{gathered}
\end{equation} Here $Q_\phi(\vz|\vx)$ is the variational conditional posterior, which is also parameterzied by a deterministic neural network called the Encoder, $E_\phi:\mathcal{X} \to \mathcal{Z}$. \\ $Q_\phi(z)=\int_{\mathbb{R}^d}Q_\phi(\vz|\vx)\  dP_d(\vx)$ is the aggregated posterior distribution imposed by the Encoder, $c:\mathcal{X} \times \mathcal{X} \to \mathbb{R}^+$ is any measurable cost function (such as Mean Square Error, MSE) and $\phi \in \Phi$, $\theta \in \Theta$ are the learnable parameters of Encoder and Decoder, respectively. The constrained optimization problem in Eq. \ref{eqn:wae_constrained_obj} translates to auto-encoding the input data with a constraint (regularizer) that the aggregated distribution imposed by the Encoder matches with a predefined latent prior distribution.  It can be equivalently written as an unconstrained problem by introducing a Lagrangian:
\begin{gather}
\begin{split}
D_{WAE}(P_d,P_{\theta^*})&=\mathop{\inf}_{\phi, \theta}\Bigg(\underbrace{\mathop{\mathbb{E}}_{P_d}\mathop{\mathbb{E}}_{Q_\phi(\vz|\vx)}\Big[c\big(\vx, D_\theta(\vz)\big)\Big]}_{\text{a}} +\\
&\qquad \lambda \cdot \underbrace{D_{Z}\big(Q_\phi(\vz),P_Z(\vz)\big) }_\text{b}\Bigg)
\end{split}
\label{eqn:wae_relaxed_obj}
% \raisetag{25pt}
\end{gather}
\noindent
Where $\lambda$ is the Lagrange multiplier\footnote{Theoretically, the objective should be optimized w.r.t. the Lagrange multiplier $\lambda$. However, in practical implementations \cite{WAE} it is considered to be a hyper-parameter.}, $D_{Z}(.)$ is any divergence measure such as Kullaback-Leibler, Jenson-Shannon or Wasserstein distance, between two distributions and $\theta^*$ represents the optimum decoder parameters. Note that objective in Eq. \ref{eqn:wae_constrained_obj} becomes feasible only when ${D_{Z}\big(Q_\phi(\vz),P_Z(\vz)\big) }$ becomes zero. Equipped with these, in Theorem \ref{thm:impossibility}, we show that when $m > n$ (most common practical case), the WAE objective (Eq. \ref{eqn:wae_constrained_obj}) does not have a feasible solution when the prior is fixed a priori to be any distribution which is supported outside of a set of countable union of all possible $n$-dimensional manifolds in an $m$-dimensional space, denoted by $\mathcal{Q}_m^n$. An example for such a prior is the
an isotropic Gaussian distribution in $\mathbb{R}^m$, which is the usual choice in most models. %\footnote{A similar result can be found for the specific case of KL divergence in \cite{MaskAAE}.}. 
\par\noindent
\begin{theorem}
\label{thm:impossibility}
If $m > n$, then the regularization term in the objective function of a WAE/AAE (Eq. \ref{eqn:wae_relaxed_obj}),  $D_{Z}(Q_\phi(\vz),P_Z(\vz)) > 0,~ \forall \phi$ and for any distributional divergence $D_Z$ when the support of $P_Z(\vz)$  $\not\in \mathcal{Q}_m^n$.
\end{theorem}
\begin{proof}
Since $f:\widetilde{\mathcal{Z}} \to \mathcal{X}$ can be approximated arbitrarily closely using a neural network (the assumption we have made earlier) and the Encoder function $E_\phi:\mathcal{X} \to \mathcal{Z}$ is also a neural network, $E_\phi \circ f: \mathbb{R}^n \to \mathbb{R}^m$ belongs to the class of composition of affine transformations and point wise non-linearities (such as rectifiers, leaky rectifiers, or smooth strictly increasing functions like sigmoid, tanh, softplus, etc.). Consequently, $\mathcal{Z}$ is a always a countable union of $n$-dimensional manifolds in a $m$-dimensional ambient space (Lemma 1 in \cite{arjovsky2017towards}). Therefore, given that the Encoder is deterministic, by definition, $Q_\phi(\vz)$ has measure zero on $\mathbb{R}^m \backslash \mathcal{Z}$, whereas the support of $P_Z(\vz)$  $\not\in \mathcal{Q}_m^n$ which implies that it has a non-zero measure outside $\mathcal{Z}$. Thus, any distributional divergence measure between $Q_\phi(\vz)$  and $P_Z(\vz)$ will assume a non-zero value, whenever $m > n$.
\end{proof}
% In fact, the above theorem can be generalized for the case with any latent prior that is not embedded fully in an  $n$-dimensional sub-manifold of the assumed $m$-dimensional space. In such cases, the WAE/AAE objective does not have any feasible solution as the constraint in Eq. \ref{eqn:wae_constrained_obj} can never be satisfied. Corollary \ref{generalized_impossibility} captures this idea formally as follows. {\color{blue} Arnab, good job, but you have to define $Q_m^n$ formally before. Also let us choose a different notation because this may be confusing. How about $\mathcal{Q}[m,n]$}
% \begin{corollary}
% \label{generalized_impossibility}
% When $m>n$, if $P_Z \not\in Q_m^n$ then $D_{Z}\big(Q_\phi(\vz),P_Z(\vz)\big) > 0,~ \forall \phi$ and for any distributional divergence $D_Z$. WAE objective has a feasible solution iff $P_Z(\vz) \in Q_m^n$, where $Q_m^n$ denotes the set of countable union of all possible $n$-dimensional manifolds in $m$-dimensional ambient space.
% \end{corollary}

% Please note that, WAE objective has feasible solutions when $m=n$. {\color{blue}Why? Some elaboration}

The above theorem asserts that it is impossible to match the aggregated latent posterior to the prior when the assumed latent dimension is more than the true latent dimension and the assumed prior has full-support, which consequently leads to bad generation quality. One possible solution for this problem is to make $m=n$ which is practically impossible because $n$ is unknown. Another way of countering this issue is to use a stochastic Encoder and fill the `extra' dimensions with external noise, however, it leads to other issues such as difficulty in Decoder training \cite{rezende2018taming}, conflict between the two terms in the objective \cite{burgess2018understanding} and non-uniqueness of solutions \cite{dai2019diagnosing}. Hence, we restrict the scope of this paper to the case of deterministic Encoder-Decoder pair and reserve the case of stochastic Encoders for our future work. 

\subsection{The Bias-Variance Trade-off}

\begin{figure*}[t!]
    \centering
    \begin{subfigure}[t]{0.24\textwidth}
        \centering
        \includegraphics[trim={0 0 25 0}, clip, keepaspectratio, width=\textwidth]{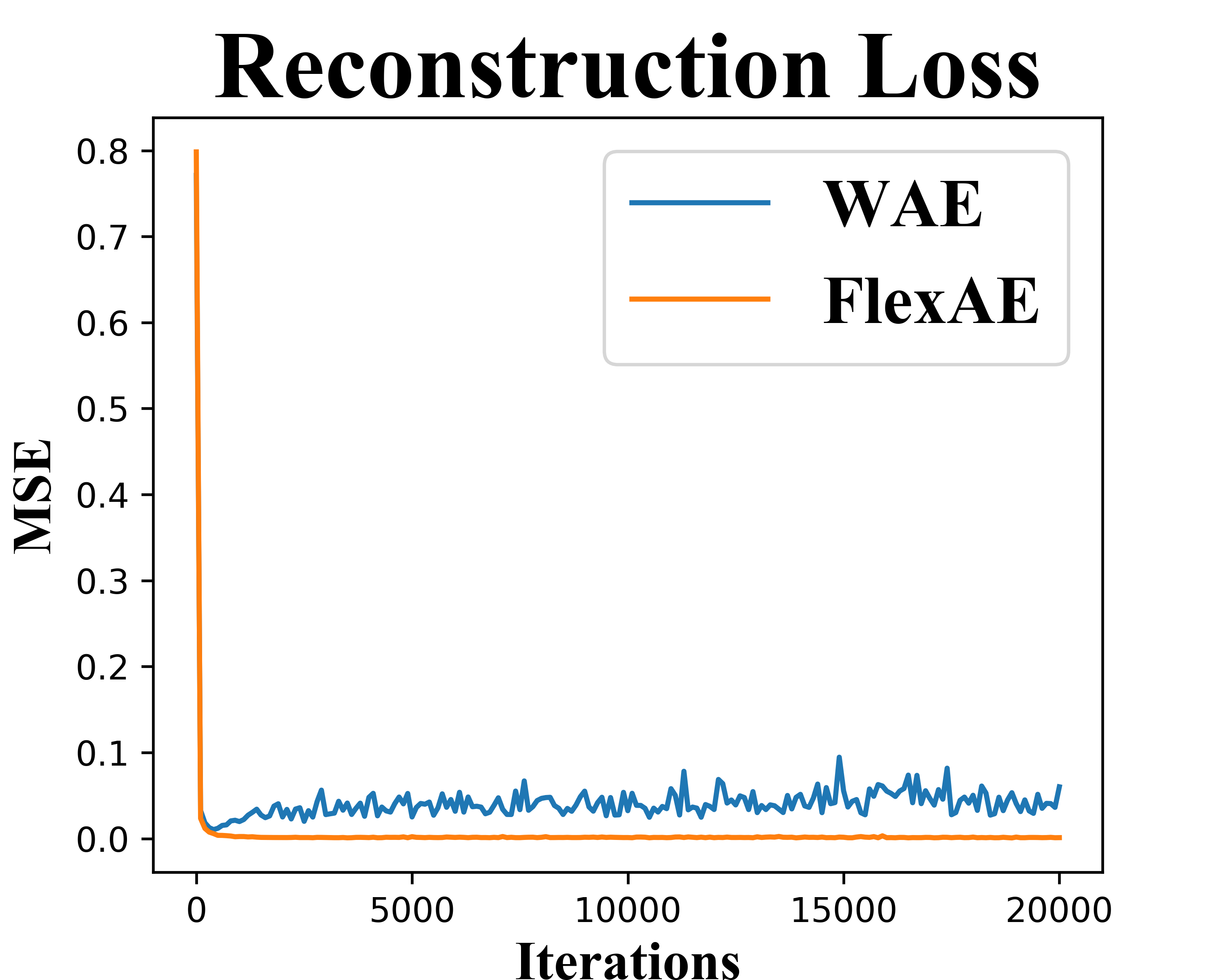}
        \caption{}
        \label{fig:recon_comp}
    \end{subfigure}%
    ~ 
    \begin{subfigure}[t]{0.24\textwidth}
        \centering
        \includegraphics[trim={0 0 25 0}, clip, keepaspectratio, width=\textwidth]{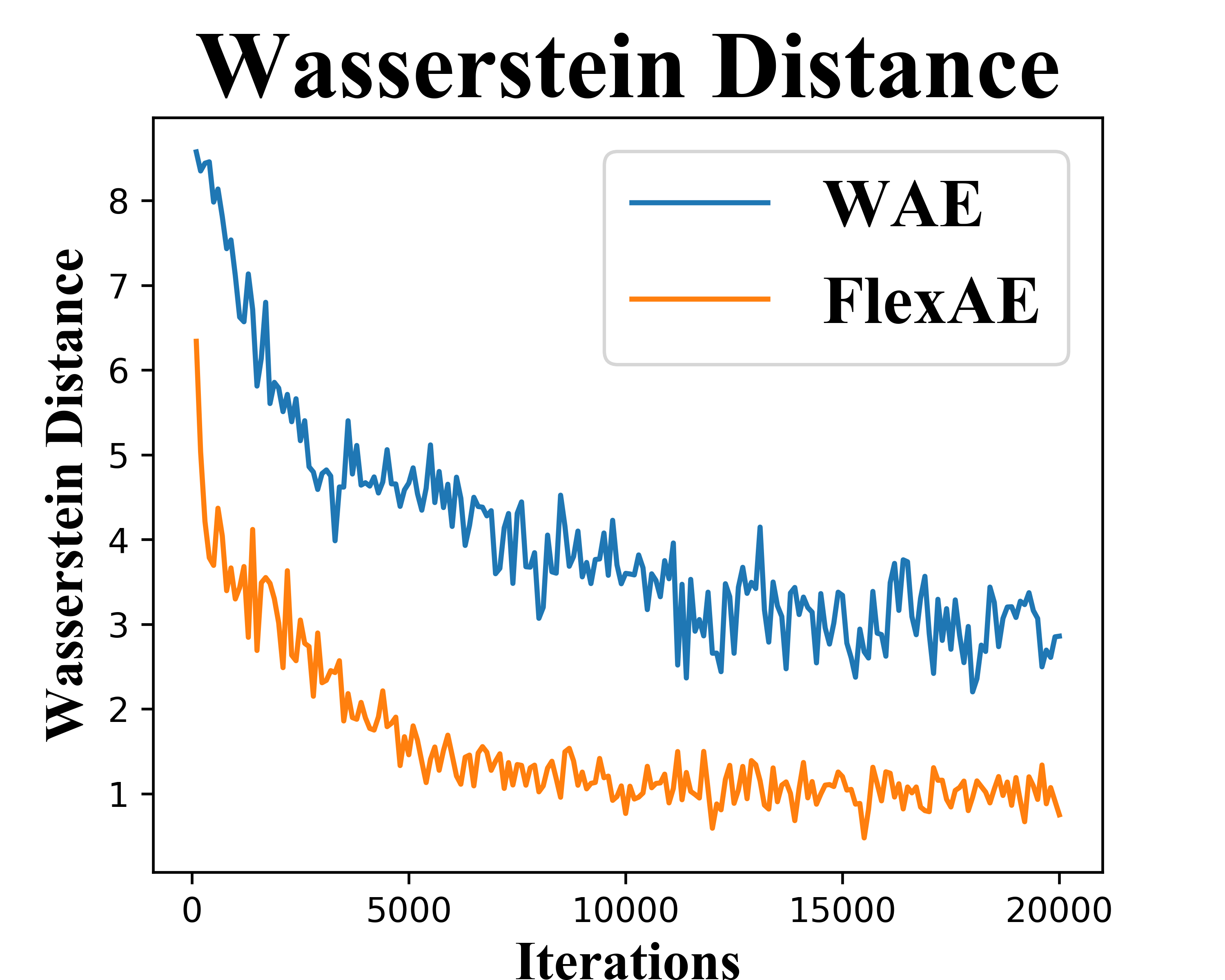}
        \caption{}
        \label{fig:wass_dist}
    \end{subfigure}%
    ~
    \begin{subfigure}[t]{0.24\textwidth}
        \centering
        \includegraphics[trim={0 0 25 0}, clip, keepaspectratio,  width=\textwidth]{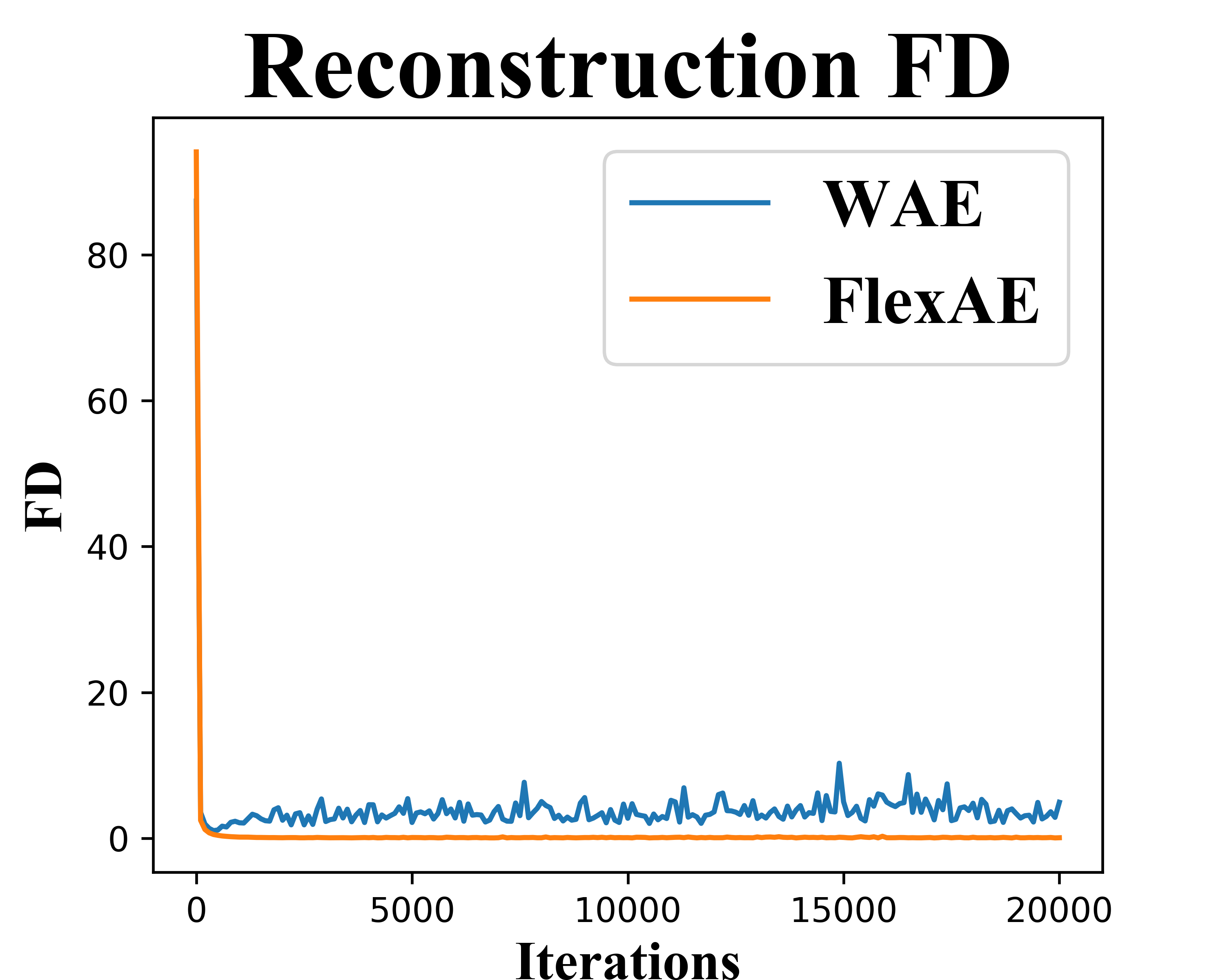}
        \caption{}
        \label{fig:recon_fd_comp}
    \end{subfigure}%
    ~
    \begin{subfigure}[t]{0.24\textwidth}
        \centering
        \includegraphics[trim={0 0 25 0}, clip, keepaspectratio,  width=\textwidth]{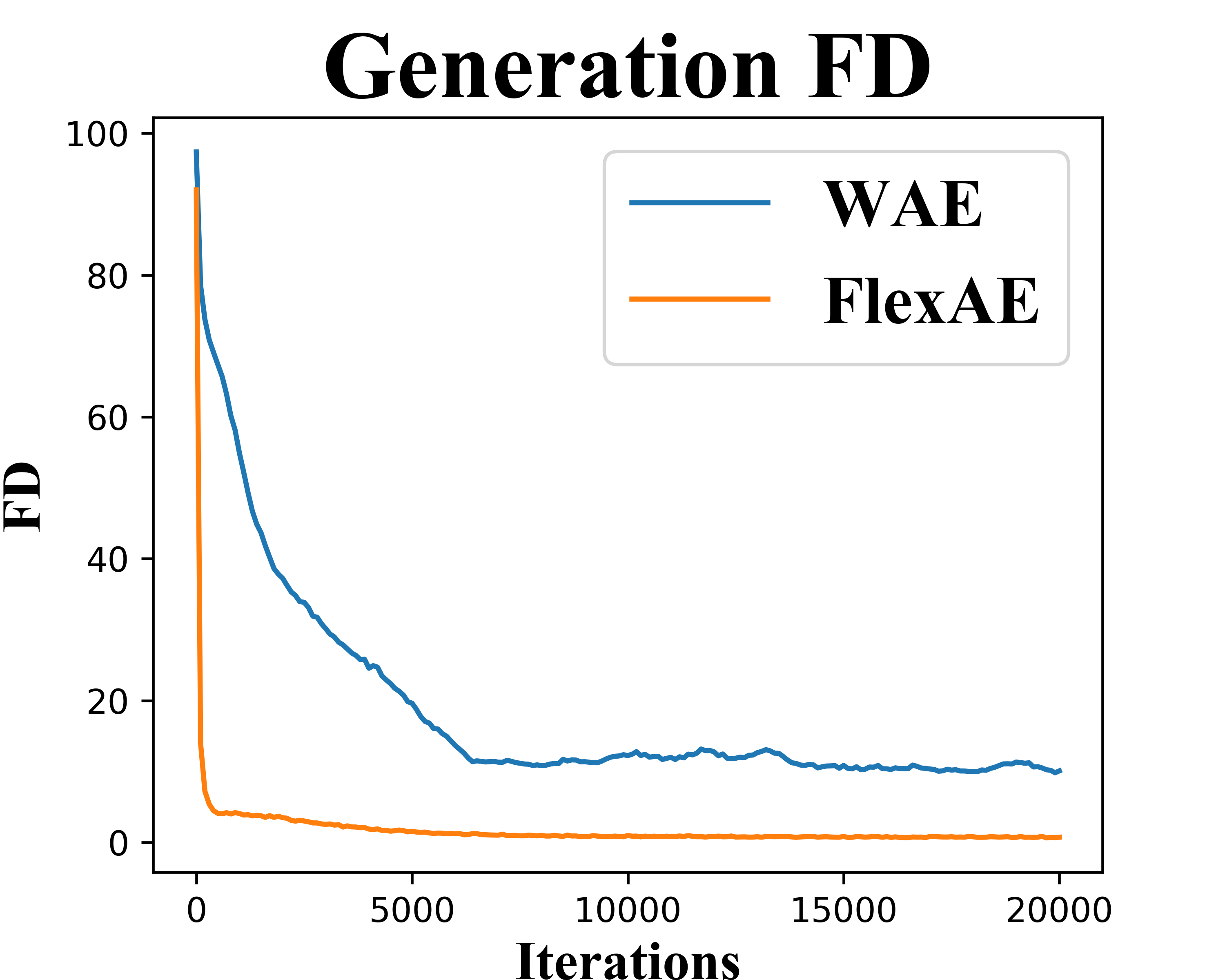}
        \caption{}
        \label{fig:gen_fd_comp}
    \end{subfigure}%
    \caption{Comparison of RAEs with  fixed and learnable latent priors on a synthetic dataset. It is seen that the Wasserstein distance between $P_z$ and $Q_z$ reduce faster in the case of FlexAE compared to a fixed prior WAE, leading to a better FD.}
    \label{fig:infeasible}
\end{figure*}

\begin{figure*}[ht!]
    \centering
    \begin{subfigure}[t]{0.19\textwidth}
        \centering
        \includegraphics[trim={6.5 6.5 6.5 6.5}, clip, keepaspectratio, width=\textwidth]{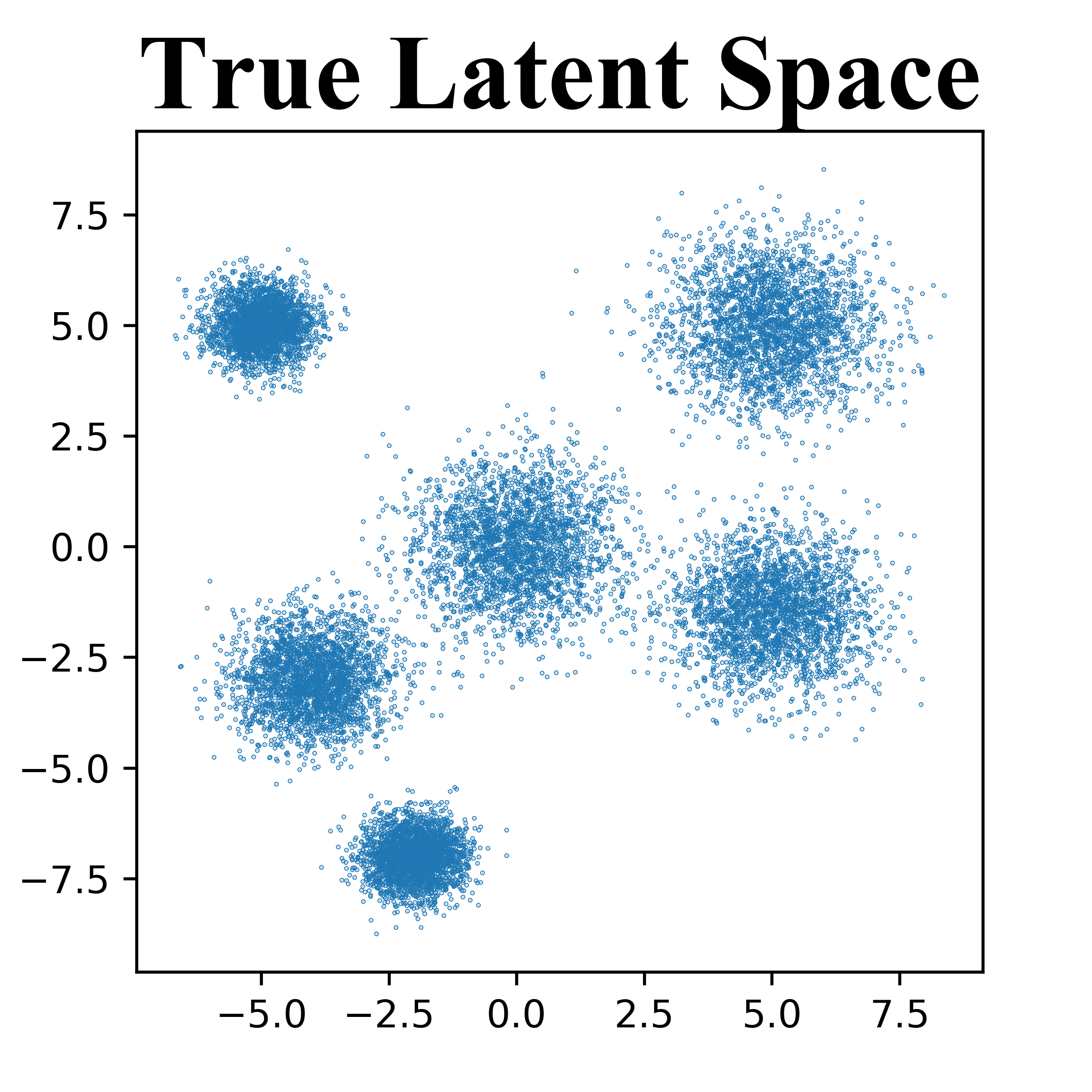}
        \caption{}
        \label{fig:true_lat_gmm}
    \end{subfigure}%
    ~ 
    \begin{subfigure}[t]{0.19\textwidth}
        \centering
        \includegraphics[trim={6.5 6.5 6.5 6.5}, clip, keepaspectratio, width=\textwidth]{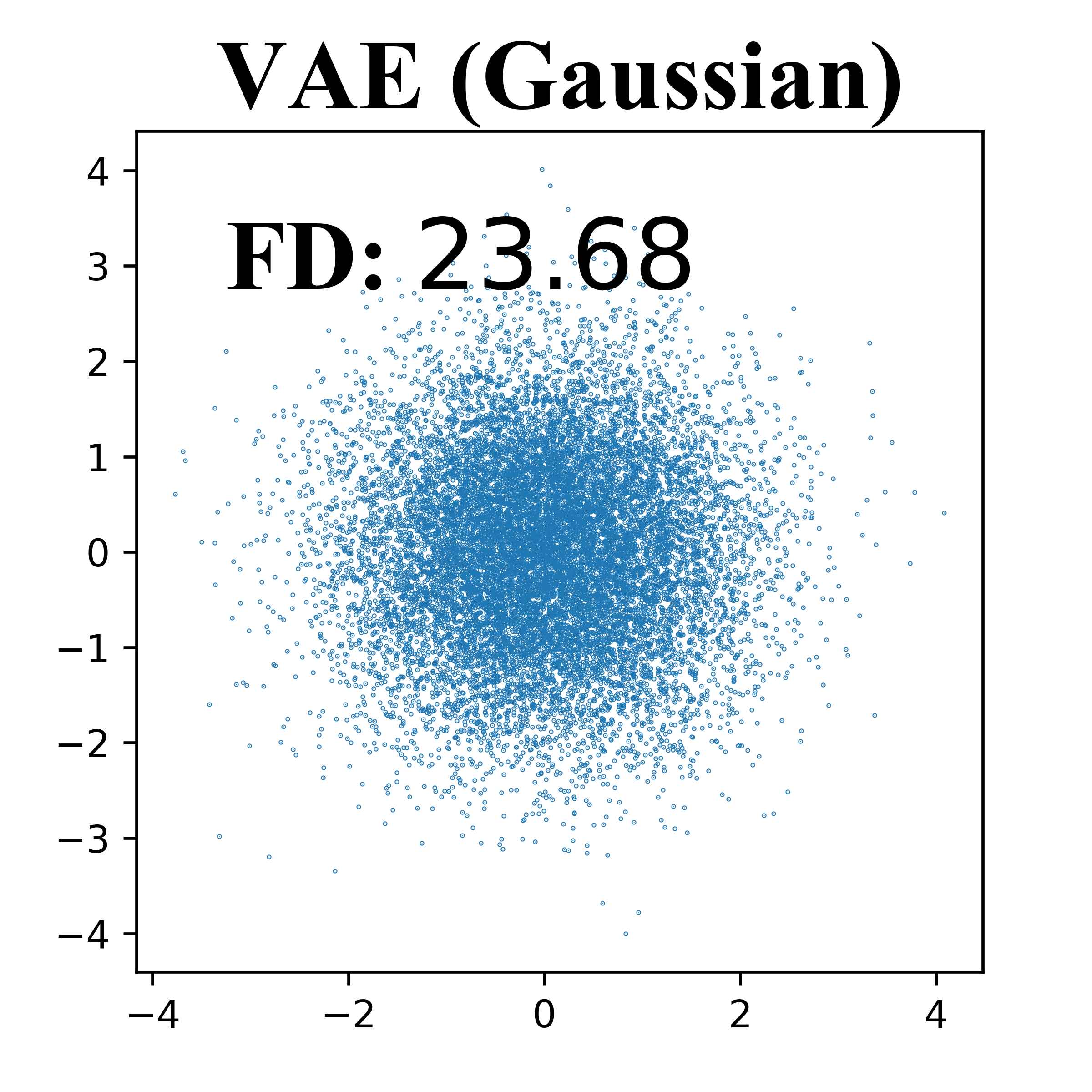}
        \caption{}
        \label{fig:vae_lat_true_gmm}
    \end{subfigure}%
    ~
    \begin{subfigure}[t]{0.19\textwidth}
        \centering
        \includegraphics[trim={6.5 6.5 6.5 6.5}, clip, keepaspectratio,width=\textwidth]{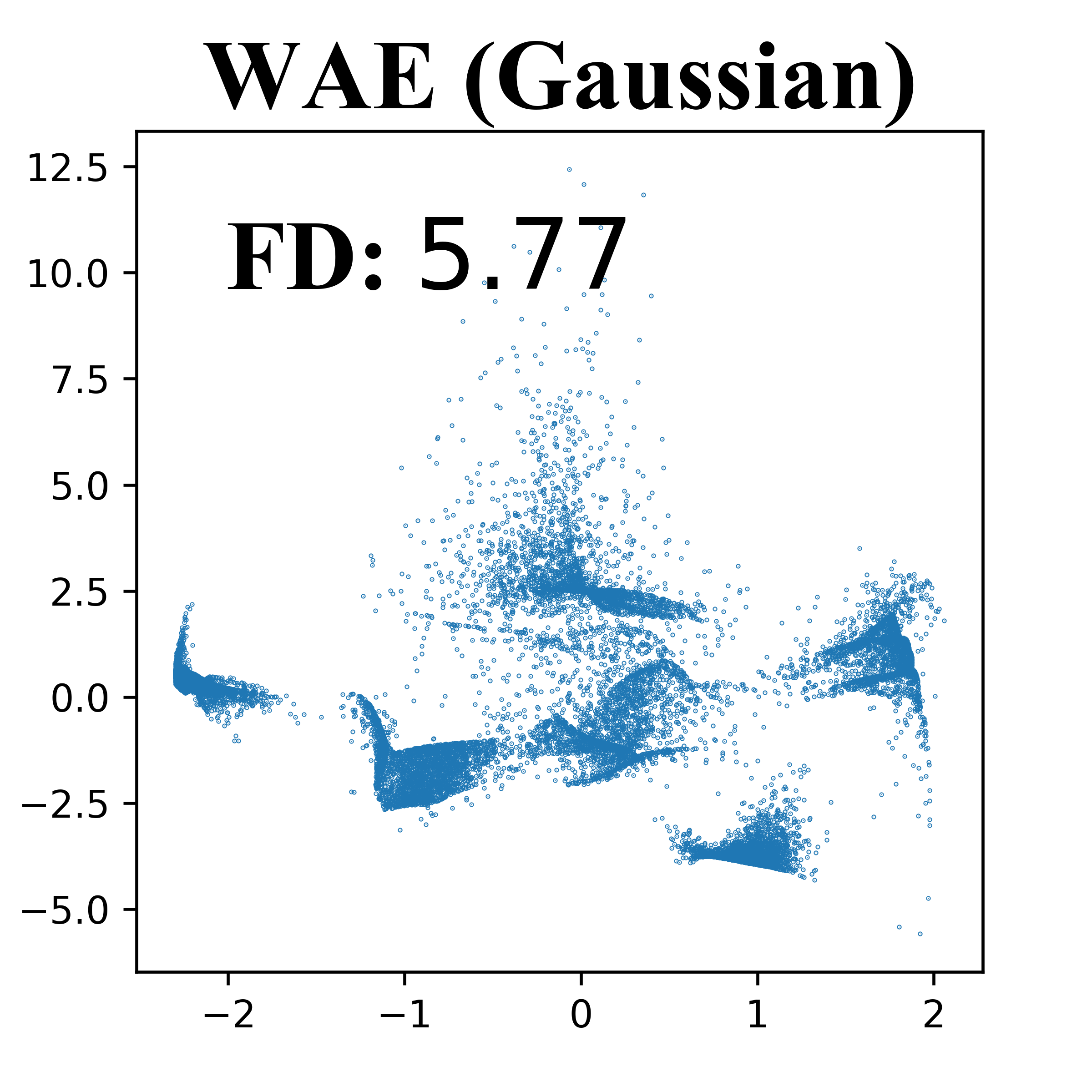}
        \caption{}
        \label{fig:wae_normal_lat_true_lat_gmm}
    \end{subfigure}%
    ~ 
    \begin{subfigure}[t]{0.19\textwidth}
        \centering
        \includegraphics[trim={6.5 6.5 6.5 6.5}, clip, keepaspectratio, width=\textwidth]{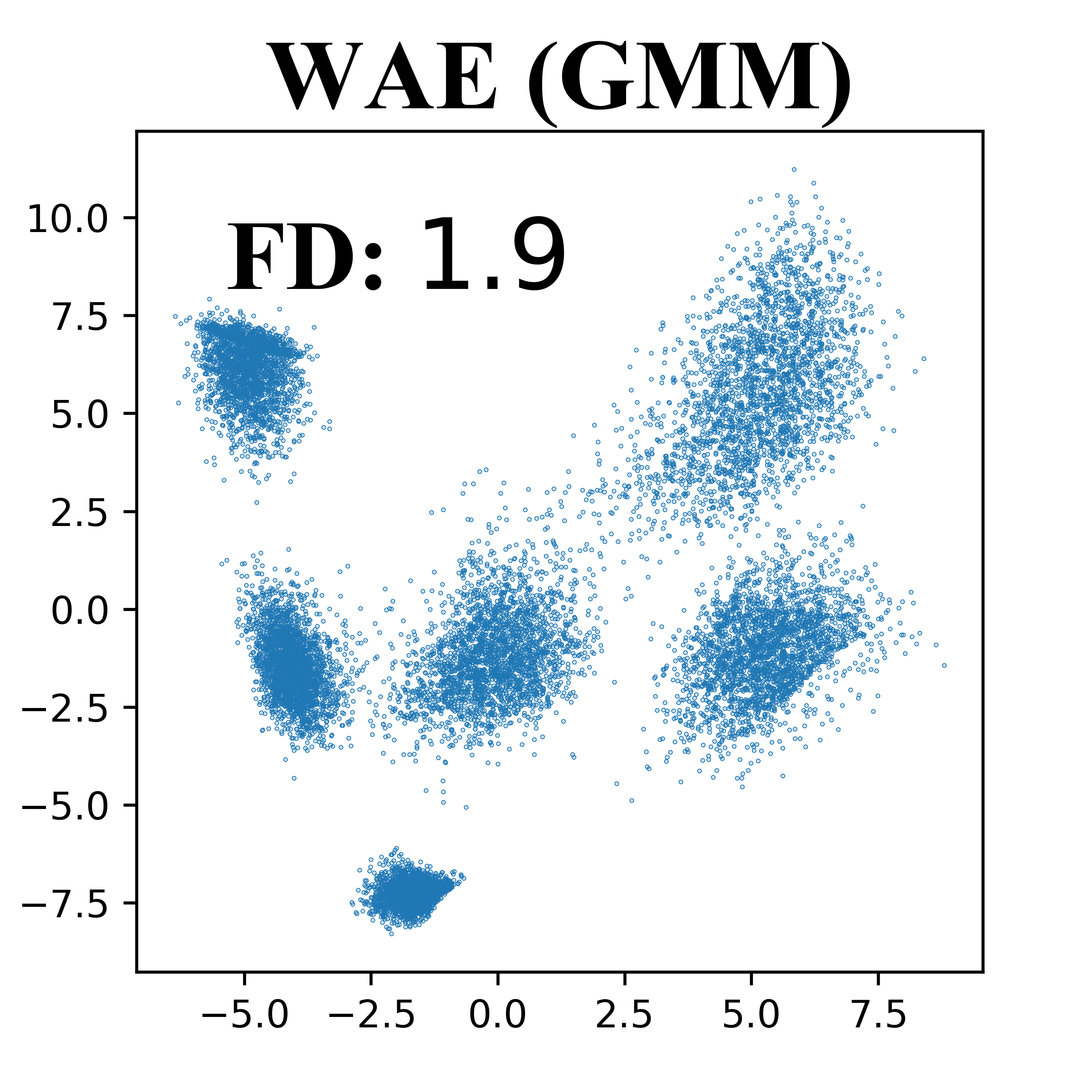}
        \caption{}
        \label{fig:wae_gmm_lat_true_gmm}
    \end{subfigure}%
    ~ 
    \begin{subfigure}[t]{0.19\textwidth}
        \centering
        \includegraphics[trim={6.5 6.5 6.5 6.5}, clip, keepaspectratio, width=\textwidth]{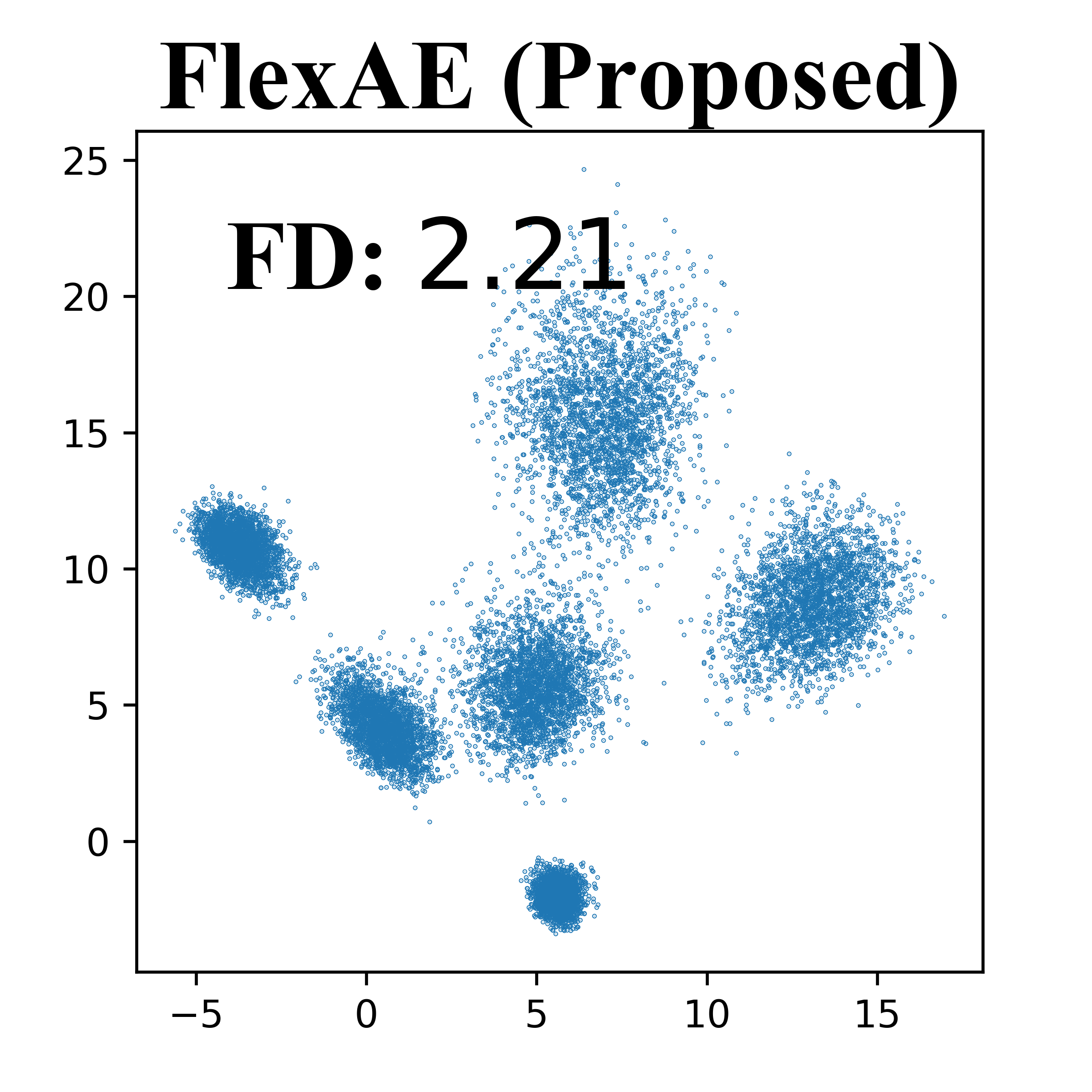}
        \caption{}
        \label{fig:flexae_lat_true_gmm}
    \end{subfigure}%
    \caption{Visualization of (a) true latent space; (b) latent space learned by the VAE \cite{kingma2013autoencoding}; (c) latent space learned by the WAE \cite{WAE} with Normal prior; (d) latent space learned by the WAE \cite{WAE} with GMM prior; and (e): latent space learned by the proposed FlexAE model, along with generation Fréchet Distance (FD) in each case. For multimodal data, model with multimodal prior (WAE-GMM) and FlexAE perform better.}
    \label{fig:fixed_prior_vs_flexible}
\end{figure*}

One way to visualize Eq. \ref{eqn:wae_relaxed_obj} is that there is an reconstruction objective (term a) and a regularizer in the form of a divergence metric (term b). Under this view, the issue mentioned in Theorem 1 could be seen to be originating because of the choice of a `wrong' regularizer. Noting this, a few recent work \cite{ghosh2020from,kumar2020regularized} have suggested to completely get rid of the latent regularizer but use an ex-post latent code sampler such as GMM, Markov-Chain Monte-Carlo (MCMC) or a GAN on the latent space after a free-form AE (only term a of Eq. \ref{eqn:wae_constrained_obj}) is trained. While this technique will theoretically avoid the problem mentioned, we argue that it imposes another practical issue when there is finite data.
\par  It is well-observed that with sufficient model capacity, a deterministic AE when trained on a finite amount of data will lead to an increased variance (over-fitting). This is because the Encoder can settle with a trivial solution for $q_{\phi^*}(\vz|\vx_i)$ which is  Dirac-deltas at all input data points $\vx_i$ \cite{rezende2018taming}. Subsequently, the post-hoc sampler (e.g. GAN) will learn to sample from finite set of Dirac-deltas \cite{arora2017gans} resulting in a non-smooth latent space and poor generalization. On the other hand, as seen in the previous section, a high bias or over-regularization will also impact the generation quality. This is the infamous bias-variance trade-off that warrants a flexible prior which could facilitate the operation of an AE-model at different points of the bias-variance points. Note that this problem may arise with models with  stochastic Encoders too. In the subsequent sections, we propose a model that can effectively handle both the issues.

\subsection{Flexibly Learning Prior: FlexAE}\label{proposed-model}
Based on the discussion so far, fixing a prior makes the optimization objective infeasible and no prior leads to over-fitting.  To alleviate these, we propose to flexibly learn the latent prior jointly with the AE-training by introducing an additional state-space in the objective of an WAE as follows: 

\begin{gather}
\begin{split}
D_{FlexAE}(P_X,P_{\theta^*})&=\mathop{\inf}_{\psi, \phi, \theta}\Bigg(\underbrace{\mathop{\mathbb{E}}_{P(\vx)}\mathop{\mathbb{E}}_{Q(\vz|\vx)}\Big[c(\vx, D_\theta(\vz)\Big]}_{\text{a}} +\\
&\qquad \lambda \cdot \underbrace{ \vphantom{\mathop{\mathbb{E}}_{P(\vx)}} D_{Z}(Q_\phi(\vz)||P_\psi(\vz)) }_\text{b}\Bigg)
\end{split}
\label{eqn:flexae_obj}
\raisetag{30pt}
\end{gather} where $P_\psi(\vz)$ is a learnable latent prior  parmaterized using a neural network called the Prior-Generator (P-GEN), $G_\psi$, that takes an $m' \ge n$ dimensional isotropic Gaussian distribution as the input and generates sample from an $m$-dimensional $P_\psi(\vz)$ (refer Fig. \ref{fig:data_gen_flexAE}). $\theta^*$ denotes the optimal decoder parameters. In our model, referred to as the Flexible AE or FlexAE, P-GEN is jointly trained with the AE to alternatively minimize the divergence measure and the reconstruction terms in Eq. \ref{eqn:flexae_obj}. Upon convergence, the output of the P-GEN forms the prior that is imposed on the latent space. In the following we show that not only that FlexAE doesn't suffer from the infeasibility problem but also helps in more flexible bias-variance trade-offs. First, it is to be noted that that  $D_{FlexAE} \leq  D_{WAE}$ and thus the new formulation does not harm the optimization. Next, the Theorem below states that the divergence measure can be brought to zero with FlexAE. 

\begin{theorem}
$\forall m' \ge n$, $D_{Z}(Q_\phi(\vz)||P_\psi(\vz))$ (term (b) in FlexAE objective (Eq. \ref{eqn:flexae_obj}) becomes zero for optimum set of parameters. 
\end{theorem}
\begin{proof}
Let, $\mathcal{P}_\psi$ denote the set of all possible manifolds on which the output of P-GEN network, $G_\psi$, may lie within $\mathbb{R}^m$. Given sufficiently large deep nets, sample size, and computation time,  $\mathcal{P}_\psi = \mathop{\cup}_{\eta \le m'} \mathcal{Q}_m^\eta$. As $m' \ge n$, this implies $\mathcal{Q}_m^n \subseteq \mathcal{P}_\psi$ which implies that $G_\psi$ can learn $P_\psi(\vz)$ to match $Q_\phi(\vz)$ driving $D_{Z}(Q_\phi(\vz)||P_\psi(\vz))$  to zero.
\end{proof}

\begin{table*}[t]
  \caption{Comparison of FID scores \cite{NIPS2017_7240} on real datasets. Lower is better.}
  \label{table:fidTable}
  \centering
  \begin{tabular}{cccccccccccc}
    \hline
    & \multicolumn{2}{c}{MNIST} & & \multicolumn{2}{c}{CIFAR10} & & \multicolumn{2}{c}{CELEBA} \\
    \cline{2-3} \cline{5-6} \cline{8-9}
    & Rec. & Gen. & & Rec. & Gen.  & & Rec. & Gen.  \\
    % \cmidrule{2-3} \cmidrule{5-6} \cmidrule{8-9}
    \hline
    VAE \cite{kingma2013autoencoding}                     & $65.10$ & $57.04$  & & $176.5$ & $169.1$ & & $62.36$ & $72.48$ \\
    $\beta$-VAE \cite{higgins2017beta}             & $7.91$  & $24.31$ & & $43.86$ & $83.59$ & & $30.06$ & $50.66$ \\
    VAE-VampPrior \cite{tomczak2017vae}          & $11.01$ & $49.75$ & & $107.33$ & $161.02$ & & $49.71$ & $64.26$ \\
    VAE-IOP \cite{takahashi2019variational}                & $8.01$ & $32.61$ & & $92.17$ & $141.92$ & & $41.52$ & $57.30$ \\
    WAE-GAN \cite{WAE}                & $8.06$ & $13.30$ & & $42.39$ & $72.90$ & & $29.34$ & $39.58$ \\
    AE + GMM (L$2$) \cite{ghosh2020from}        & $8.69$  & $12.14$ & & $41.45$ & $70.97$ & & $30.16$ & $43.89$ \\
    RAE + GMM (L$2$) \cite{ghosh2020from}       & $6.15$ & $7.30$ & & $40.48$ & $69.24$ & & $29.05$ & $35.30$ \\
    VAE + FLOW \cite{kingma2016improved}             & $8.62$ & $20.17$ & & $43.87$ & $73.28$ & & $36.31$ & $42.39$ \\
    $\text{InjFlow}^{ln}$ \cite{kumar2020regularized}   &  $7.40$ & $35.96$ & & $40.11$ & $78.78$ & & $27.93$ & $47.70$ \\
    $\text{InjFlow}^{ln}$ + GMM \cite{kumar2020regularized}  &  $7.40$ & $9.93$ & & $40.11$ & $68.26$ & & $27.93$ & $40.23$ \\
    $2$-S VAE \cite{dai2019diagnosing}              & $6.38$ & $7.41$ & & $47.03$ & $86.15$ & & $29.38$ & $37.85$ \\
    MaskAAE \cite{MaskAAE}                & $8.46$ & $10.52$ & & $58.40$ & $71.90$ & & $35.75$ & $40.49$ \\
    FlexAE (Proposed)                 & $\boldsymbol{4.33}$ & $\boldsymbol{4.69}$ & & $\boldsymbol{39.91}$ & $\boldsymbol{62.66}$ & & $\boldsymbol{21.17}$ & $\boldsymbol{25.96}$ \\
    \hline
  \end{tabular}
  \vspace{-2mm}
\end{table*}

\begin{table}[!htbp]
  \caption{Comparison of Precision/Recall scores \cite{Pre_rec} on real datasets. Higher is better.}
  \label{table:prdTable}
  \centering
  \resizebox{\columnwidth}{!}{
  \begin{tabular}{cccccc}
    \hline
                    & MNIST       & & CIFAR$10$   & &
                    CELEBA      \\
    \hline
    VAE \cite{kingma2013autoencoding}             & $0.69/0.76$ & & $0.23/0.47$ & & $0.47/0.58$ \\
    $2$S-VAE \cite{dai2019diagnosing}        & $0.97/0.98$ & & $0.47/0.76$ & & $0.75/0.72$ \\
    RAE + GMM (L$2$) \cite{ghosh2020from} & $0.98/0.98$ & & $0.61/0.87$ & & $0.74/0.75$ \\
    MaskAAE \cite{MaskAAE}          & $0.94/0.96$ & & $0.58/0.83$ & & $0.59/0.68$ \\
    FlexAE (Proposed)          & $\boldsymbol{0.99/0.99}$ & & $\boldsymbol{0.68/0.85}$ & & $\boldsymbol{0.89/0.88}$ \\
    \hline
  \end{tabular}
  }
  \vspace{-2mm}
\end{table}

\par
For implementation, we use MSE for $c$ in term (a) of Eq. \ref{eqn:flexae_obj}. $D_Z$, in principle can be chosen to be any distributional divergence such as Kullback-Leibler divergence (KLD), Jensen–Shannon divergence (JSD), Wasserstein Distance and so on. In this work, we propose to use Wasserstein distance and utilize the principle laid in \cite{arjovsky2017wasserstein, gulrajani2017improved}, to optimize the divergence term (b) in Equation \ref{eqn:flexae_obj}). The loss functions used for different blocks of FlexAE are as follows:
\begin{enumerate}
    \item {Likelihood Loss - Realization of Term a in Eq. \ref{eqn:flexae_obj}:
    \begin{equation}
    L_{AE} = \frac{1}{s}\sum_{i=1}^{s}||\vx^{(i)}-D_\theta(E_\phi(\vx^{(i)}))||^2
    \label{eqn:lle_loss_mse}
    \end{equation}}
  \item {Wasserstein Loss - We use Wasserstein distance \cite{arjovsky2017wasserstein} for $D_Z$ (Term b Eq. \ref{eqn:flexae_obj}):}
    \begin{equation}
        \begin{split}
            L_{Critic} &= \frac{1}{s}\sum_{i=1}^{s}C_\kappa(\hat{\vz}^{(i)}) - \frac{1}{s}\sum_{i=1}^{s}C_\kappa (\vz^{(i)}) \\
            &\qquad 
            + \frac{\beta}{s}\sum_{i=1}^{s}\big(\lvert\lvert\nabla_{\vz_{avg}}^{(i)}C_\kappa (\vz_{avg}^{(i)})\lvert\lvert - 1\big)^2
        \end{split}
        \label{eqn:critic_loss}
    \end{equation}
    \begin{equation}
            L_{Gen} = -\frac{1}{s}\sum_{i=1}^{s}C_\kappa(\hat{\vz}^{(i)})
            \label{eqn:gen_loss}
    \end{equation}
    \begin{equation}
            L_{Enc} = \frac{1}{s}\sum_{i=1}^{s}C_\kappa(\vz^{(i)})
            \label{eqn:enc_loss}
    \end{equation}
  
\end{enumerate}
Where, $\vz^{(i)} = E_{\phi}(\vx^{(i)})$, $\hat{\vz}^{(i)} = G_{\psi}(\vn^{(i)})$ and $\vn^{(i)} \sim \mathcal{N}(0, I)$. $\vz_{avg}^{(i)} = \alpha\vz^{(i)} + (1-\alpha)\hat{\vz}^{(i)}$,  $\alpha,\beta$ are hyper parameters, with $\alpha \sim \mathcal{U}[0, 1]$, and $\beta$ as in  \cite{gulrajani2017improved}. $E_{\phi}, D_\theta, G_\psi$, and $C_\kappa$ denote the encoder, decoder, latent generator and critic respectively. Also with the cost $c$ chosen as MSE (Eq. \ref{eqn:lle_loss_mse}), the LHS of the objective (Eq. \ref{eqn:flexae_obj}) becomes $2$-Wasserstein distance.\par
Figure \ref{fig:infeasible} demonstrates the benefit of FlexAE over a WAE,  where the performance of both the models is shown on a synthetic data: $\mathcal{\widetilde{Z}} = \mathbb{R}^5$ and $f:\mathbb{R}^5 \to \mathbb{R}^{128}$ is an arbitrary multi-layer perceptron (details in the Tech. Appendix). It is seen that, when $m=50$, Wasserstein distance between $P_z$ and $Q_z$ reduce faster and reaches much lower values in the case of FlexAE compared to a fixed prior WAE, leading to a better Fr\'{e}shet Distance on the generated data. \par 
Further, the P-GEN network allows FlexAE to better trade-off between over-fitting and under-fitting: former is addressed by having a regularizer in the form of finite capacity P-GEN, and latter is avoided by having a learnable P-GEN with sufficient capacity to represent the desired distribution (see Figure \ref{fig:bias_variance}). Figure \ref{fig:flexae_lat_true_gmm} demonstrates this effect where it is seen that the latent space learned by a FlexAE and a WAE with a GMM prior, on a synthetic data (details in  the Tech. Appendix) results in better generation as compared to the models with fixed uni-modal Gaussian priors (Note that this figure is to show that a flexible prior helps in learning but not to show impossibility). Finally, the data generation in FlexAE happens as follows - (i) sample from a primitive (Gaussian) distribution and pass it through the P-GEN to sample a point from the latent space  $p_{\psi}(\vz)$, (ii) input the latent sample through the Decoder to generate a data sample. Algorithm for training FlexAE can be found in the Tech. Appendix.
\section{Experiments and Results}
\label{sec:expt} 

% \begin{figure*}[!htbp]
%     \centering
%     \begin{subfigure}[t]{0.33\textwidth}
%         \centering
%         \includegraphics[trim={10 12 10 12}, clip, keepaspectratio,width=\linewidth]{assets/mnist_bv.png}
%         \caption{}
%         \label{fig:mnist_bv}
%     \end{subfigure}%
%     ~ 
%     \begin{subfigure}[t]{0.33\textwidth}
%         \centering
%         \includegraphics[trim={10 12 10 12}, clip, keepaspectratio, width=\linewidth]{assets/cifar10_bv.png}
%         \caption{}
%         \label{fig:cifar10_bv}
%     \end{subfigure}%
%     ~ 
%     \begin{subfigure}[t]{0.33\textwidth}
%         \centering
%         \includegraphics[trim={10 12 10 12}, clip, keepaspectratio, width=\linewidth]{assets/celeba_bv.png}
%         \caption{}
%         \label{fig:celeba_bv}
%     \end{subfigure}
%     \caption{Bias-Variance}
%     \label{fig:bias_variance}
% \end{figure*}

\begin{figure*}[!htbp]
    \centering
    \begin{subfigure}[t]{0.24\textwidth}
        \centering
        \includegraphics[trim={2 2 2 2}, clip, keepaspectratio,width=\linewidth]{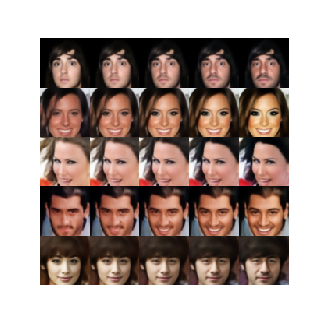}
        \caption{}
        \label{fig:train_pos}
    \end{subfigure}%
    ~ 
    \begin{subfigure}[t]{0.24\textwidth}
        \centering
        \includegraphics[trim={2 2 2 2}, clip, keepaspectratio, width=\linewidth]{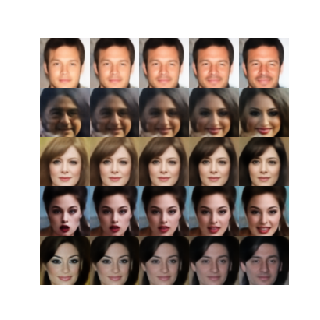}
        \caption{}
        \label{fig:test_neg}
    \end{subfigure}%
    ~ 
    \begin{subfigure}[t]{0.24\textwidth}
        \centering
        \includegraphics[trim={2 2 2 2}, clip, keepaspectratio, width=\linewidth]{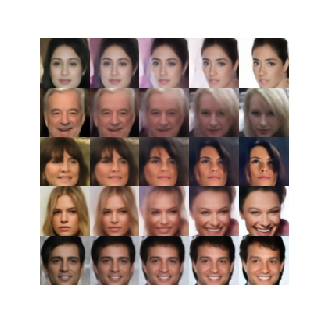}
        \caption{}
        \label{fig:test_linear}
    \end{subfigure}%
    ~ 
    \begin{subfigure}[t]{0.24\textwidth}
        \centering
        \includegraphics[trim={6.5 6.5 6.5 6.5}, clip, keepaspectratio, width=\linewidth]{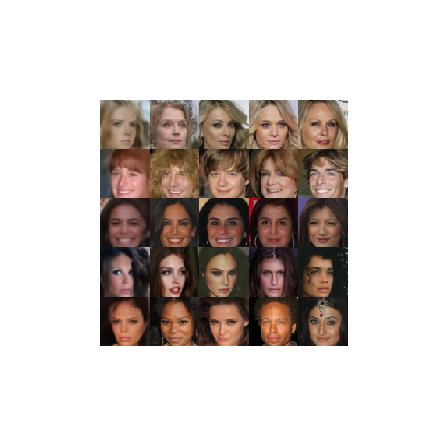}
        \caption{}
        \label{fig:nn}
    \end{subfigure}
    \caption{Interpolations in the latent space of FlexAE on CelebA. Each row in (a) and (b) presents manipulation of a particular face attribute (Big Nose, Heavy Makeup, Black Hair, Smiling, Male). The central image of each row of (a) and (b) is a true image from the train and test split with and without the attribute respectively.  Each row in (c) represents linear interpolation in the latent space between two randomly selected test samples in the first and the last entry. The first image in each row in (d) shows randomly generated samples using FlexAE and the next four entries are the four nearest neighbours from  training data.}
    \label{fig:interpolation_nn}
    \vspace{-5mm}
\end{figure*}
We consider three real-world datasets: MNIST \cite{MNIST}, CIFAR-$10$ \cite{krizhevsky2009learning}, and CelebA \cite{liu2015faceattributes} for our four set of experiments.
\subsection{Baseline Experiments}
\label{sec:expt-sota_comparison}
{\bf Methodology:}
The first task is to evaluate the FlexAE as a generative model. We use
Fr\'echet Inception Distance, (FID) \cite{NIPS2017_7240}, one of the most commonly used evaluation methods as it correlates well with human visual perception \cite{lucic2018gans}. However, as observed in \cite{Pre_rec}, FID, being uni-dimensional, fails to distinguish between different cases of failure (poor sample quality and limited variation in the samples). Thus, we also report the precision and recall metrics described in \cite{Pre_rec} along with FID, both of which are computed between the generated and the real test images. We compare FlexAE with a number of SoTA AE-based generative models that cover a broad class namely, VAE \cite{kingma2013autoencoding}, $\beta$-VAE \cite{higgins2017beta}, VAE-VampPrior \cite{tomczak2017vae}, VAE-IOP \cite{takahashi2019variational}, WAE \cite{WAE}, a plain with AE post-hoc GMM, RAE+GMM \cite{ghosh2020from}, VAE+Flow \cite{kingma2016improved}, InjFlow \cite{kumar2020regularized}, 2-stage VAE \cite{dai2019diagnosing} and MaskAAE \cite{MaskAAE}, with same architectures (see Tech. Appendix).\par\noindent
{\bf Results:}
Table \ref{table:fidTable} compares the average reconstruction and generation FID scores (lower is better) of FlexAE over three executions (variance $\pm 0.59$) with other AE-based generative models. It is seen that while models with parametric learnable  priors (VampPrior, IOP, Flow) offer some improvement over the naive VAE, they are non optimum. It is also seen that complex prior models tend to over fit more (gap between the generation and reconstruction FIDs). Further, having the ``right'' dimensional latent space seems to have significant impact (2SVAE, MaskAAE). A relatively better performance of RAE+GMM, InjFlow shows that while absence of prior imposition will reduce the bias, it might lead to over fitting. Finally, FlexAE offers the best performance on all three datasets as compared to other AE based generative models and its performance on MNIST and CelebA are comparable to that of the GANs. A similar trend is observed with the Precision/Recall numbers in  Table \ref{table:prdTable} (We only use better SoTA models for comparison). It is seen that FlexAE offers significantly better numbers in terms of both Precision and Recall confirming its effectiveness in generating samples that are of both high quality and variety.

\subsection{Effect of Latent Space Dimensionality}
\textbf{Methodology:}
To study how the latent space dimensionality affects the generation quality of the RAE, we train  FlexAE and WAE models with varying $m$.
\par\noindent \textbf{Results:} As presented in Table \ref{table:m_vs_fid}, with increasing $m$, the reconstruction FID decreases for both WAE and FlexAE models. However, the generation FID of WAE models increases with $m$. While generation FID of FlexAE remains almost constant. This shows that FlexAE can achieve better optimum irrespective of the chosen model dimensionality.
\begin{table}[!htbp]
\vspace{-2mm}
\caption{Variation of FID w.r.t. bottleneck layer dimension, $m$. For MNIST, $m_b=20$ and for CELEBA $m_b = 64$.}
  \label{table:m_vs_fid}
  \resizebox{\columnwidth}{!}{
  \centering
  \begin{tabular}{cccccccccccc}
    \hline
    $m$ & \multicolumn{5}{c}{MNIST} & & \multicolumn{5}{c}{CELEBA} \\
    \cline{2-6} \cline{8-12}
    & \multicolumn{2}{c}{Rec.} & & \multicolumn{2}{c}{Gen.} & & \multicolumn{2}{c}{Rec.} & & \multicolumn{2}{c}{Gen.} \\
    & WAE & FlexAE & & WAE & FlexAE & & WAE & FlexAE & & WAE & FlexAE \\
    \cline{2-3} \cline{5-6} \cline{8-9} \cline{11-12}
    $m_b$ & $7.16$ & $5.59$ & & $14.32$ & $5.99$ & & $30.12$ & $24.45$ & & $40.23$ & $26.09$ \\
    $2m_b$ & $5.17$ & $3.22$ & & $23.11$ & $4.22$ & & $29.34$ & $21.17$ & & $39.58$ & $25.96$ \\
    $4m_b$ & $3.12$ & $1.42$ & & $35.20$ & $5.92$ & & $28.21$ & $21.13$ & & $49.34$ & $28.36$ \\
    \hline
  \end{tabular}
  }
  \vspace{-2mm}
\end{table}

\subsection{Bias-Variance Trade-off}
\label{sec:expt-bias_variance}
{\bf Methodology:}
To  evaluate our claims on the Bias-Variance trade-off, we repeat the generation experiments by varying the capacity of the prior generator (P-GEN) from very low capacity to very high capacity (details of models in the Tech. Appendix), on a small subset of training data ($5000$ samples). Sub-sampling is to ensure that effect of bias-variance is apparent. Models of huge capacity are needed to observe similar effects of the entire dataset. 

\par\noindent {\bf Results:} 
Figure \ref{fig:bias_variance} shows that there is a performance drop at either sides of moderate capacity models (Model 3 or Model 4). As the capacity of the P-GEN increases, the reconstruction FID decreases while generation FID increases, signalling over fitting. A reverse observation could be made about the high-bias low capacity models. This confirms our hypothesis of existence of a Bias-Variance curve. Please note, in Experiment 1 and 2, the architecture of the P-GEN was kept fixed across all datasets. Therefore, even though the mere architectural choice for the P-GEN imposes a bias, the flexibility  (needed for trade off) is ensured in terms of the parameters of P-GEN.

\begin{figure}[!htbp]
    \vspace{-2mm}
    \centering
    \includegraphics[trim=20 2 20 2, clip, keepaspectratio, height=2 in]{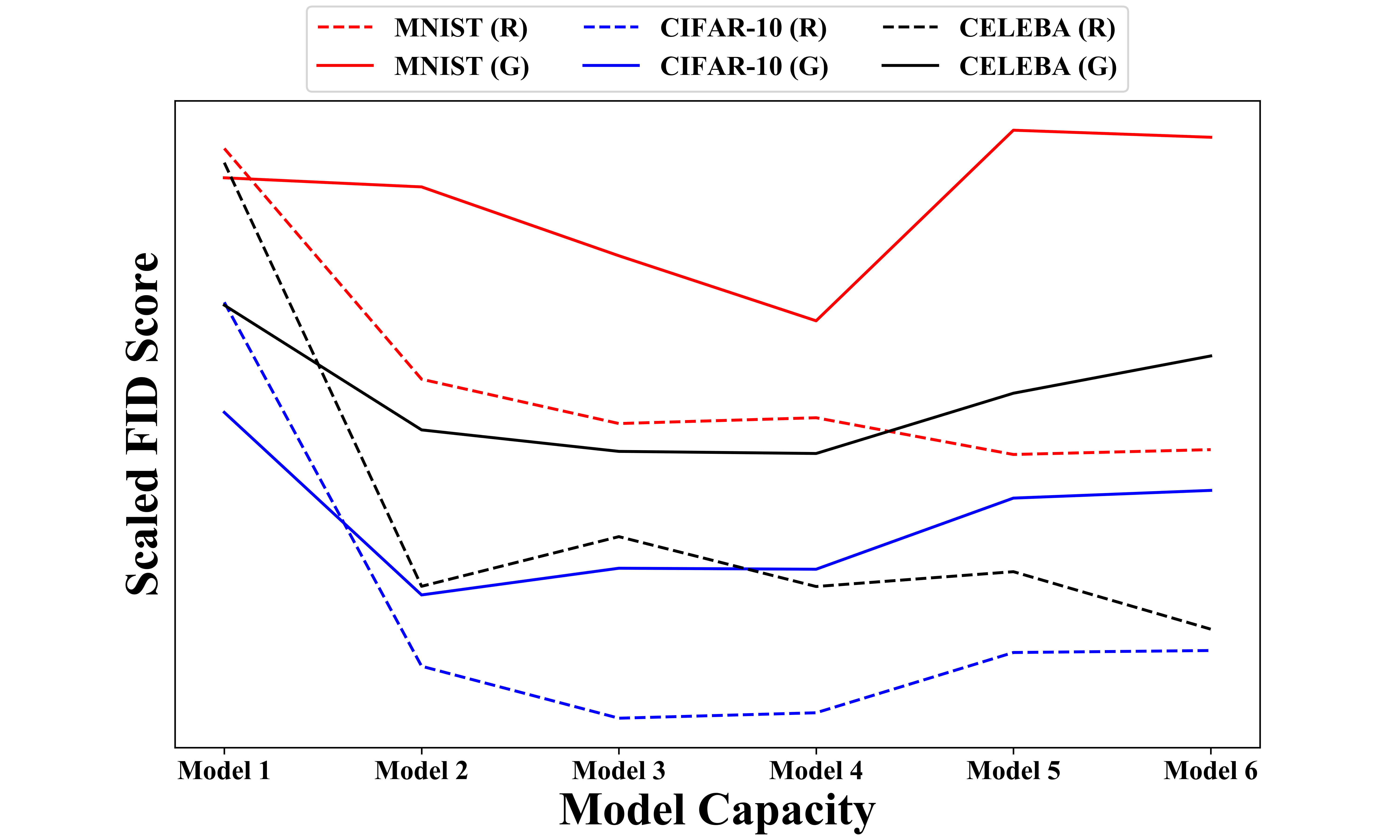}
    \caption{Variation of reconstruction and generation FID scores on limited training datasets with varying P-GEN capacity, demonstrating bias-variance trade-off. Models (1-6) are presented in increasing order of capacity.}
    \label{fig:bias_variance}
    \vspace{-4mm}
\end{figure}

\subsection{Smoothness of the Latent Space}
\label{sec:expt-smooth_gen}
%\par {\bf Experiment 3:} 
{\bf Methodology:}
%In the third set of experiments, 
To ascertain the smoothness of the learned latent space and that FlexAE doesn't over fit, we conduct a few qualitative experiments on the CelebA dataset: (i) Generation by transitions in the latent space along the direction of a particular attribute, (ii) transitions in the latent vectors between two generated samples and (iii) plot of the Nearest neighbour samples for a given generated image, from the training set, with interpolations done in the latent space.

\par\noindent {\bf Results:} The outcome of these experiments are shown in Figure \ref{fig:train_pos}, \ref{fig:test_neg}, \ref{fig:test_linear} and \ref{fig:nn}. Each row in (a) and (b) presents manipulation of a particular face attribute (Big Nose, Heavy Makeup, Black Hair, Smiling, Male). The middle image in each row of (a) corresponds to a training sample with the attribute present and the middle image of a row in (b) represents a sample without the attribute. Each row in (c) represents linear interpolation in the latent space between two randomly selected test samples in the first and the last column. The interpolation results presented in (a), (b), and (c) clearly depicts the smoothness of the learnt latent space of FlexAE as it provides smooth transition between any two random images. The first image in each row in (d) shows a randomly generated sample using FlexAE and the next four entries are the four nearest neighbours from the training split. Visual dissimilarity between any generated image and its nearest neighbours from the training split confirms that FlexAE has not merely memorized the training set. (cf. Tech. Appendix for more qualitative results).

\section{Conclusion}
In this paper, we systematically studied the effect of the latent prior on the AE-based generative models. We demonstrated that fixing any kind of prior in a data-agnostic way is detrimental to the performance. We also showed that with finite data, there exists a bias-variance trade-off with imposition of any prior on the latent space. We proposed a model called the FlexAE that can potentially operate at different points of the bias-variance curve, and empirically demonstrated its efficacy.
\section{Ethical Impact}
\small{
Our work falls broadly in the area of techniques dealing with automatic data generation. We have touched both the theoretical as well as experimental aspects of this problem in our work. We believe our results/findings should be available for all scientific community for furthering research and development in this area, independent of their background (e.g., race, caste,creed,gender,nationality etc.). Datasets used in our experiments are pretty standard, and we do not think our work poses any specific ethical questions or creates potential biases against any particular groups.}

\title{To Regularize or Not To Regularize? The Bias Variance Trade-off in Regularized AEs \\\vspace{4mm} \large{Technical Appendix}}

\author{}
\date{}

\maketitle

\section{Details of Datasets}
In this section, we describe the steps involved in synthetic dataset creation and provide relevant details (such as dimension, number of training and test examples and so on) of the synthetic datasets and the real datasets used in our work to experimentally validate our theoretical claims.
\subsection{Synthetic Datasets}
Synthetic data has been generated using a two step process. The steps involved in creating the dataset (corresponding to Figure 3 in the main paper) where the true latent space is GMM are listed below.
\begin{enumerate}
    \item \textbf{Step 1:} Six two-dimensional Gaussian distributions are used to generate true latent space of the synthetic dataset. $z_{k1}^{(i)}$ and $z_{k2}^{(i)}$ denotes the $1^{st}$ and the $2^{nd}$ dimensions of the $i^{th}$ sample from the $k^{th}$ distribution respectively. The distributions are as mentioned below:\\\\
    $$\begin{bmatrix}z_{11}^{(i)} \\ z_{12}^{(i)}  \end{bmatrix} \sim \mathcal{N}\Bigg(\begin{bmatrix}0\\ 0\end{bmatrix}, \begin{bmatrix}1 & 0 \\ 0  & 1 \end{bmatrix}\Bigg)$$\\
    $$\begin{bmatrix}z_{21}^{(i)} \\ z_{22}^{(i)}  \end{bmatrix} \sim \mathcal{N}\Bigg(\begin{bmatrix}5\\ 5\end{bmatrix}, \begin{bmatrix}1 & 0 \\ 0  & 1 \end{bmatrix}\Bigg)$$\\
    $$\begin{bmatrix}z_{31}^{(i)} \\ z_{32}^{(i)}  \end{bmatrix} \sim \mathcal{N}\Bigg(\begin{bmatrix}-5\\ 5\end{bmatrix}, \begin{bmatrix}0.5 & 0 \\ 0  & 0.5 \end{bmatrix}\Bigg)$$\\
    $$\begin{bmatrix}z_{41}^{(i)} \\ z_{42}^{(i)}  \end{bmatrix} \sim \mathcal{N}\Bigg(\begin{bmatrix}5\\ -1.5\end{bmatrix}, \begin{bmatrix}0.95 & 0 \\ 0  & 0.95 \end{bmatrix}\Bigg)$$\\
    $$\begin{bmatrix}z_{51}^{(i)} \\ z_{52}^{(i)}  \end{bmatrix} \sim \mathcal{N}\Bigg(\begin{bmatrix}-2\\ -7\end{bmatrix}, \begin{bmatrix}0.5 & 0 \\ 0  & 0.5 \end{bmatrix}\Bigg)$$\\
    $$\begin{bmatrix}z_{61}^{(i)} \\ z_{62}^{(i)}  \end{bmatrix} \sim \mathcal{N}\Bigg(\begin{bmatrix}-4\\ -3\end{bmatrix}, \begin{bmatrix}0.75 & 0 \\ 0  & 0.75 \end{bmatrix}\Bigg)$$
    
    \item \textbf{Step 2:} Next, a three layer MLP is used to map the two-dimensional points obtained from Step 1 to $128$-dimensional data points. Each layer consists of 128 neurons and non-linearity used in each layer is $\tanh, \exp, \tanh$ respectively. Weight and bias parameters of each layer is drawn randomly from the following three distributions respectively: $\mathcal{N}(0, 0.05), \mathcal{N}(0, 0.2), \mathcal{N}(0, 0.1)$.
\end{enumerate}

The dataset related to Figure 2 in the main paper is also generated using a 6 component GMM latent space. However, the true latent has dimension 5 and synthetic data has dimension 128 as before. The mean and variance of the Gaussian components of the true latent space are listed below.
$$\begin{bmatrix}z_{11}^{(i)} \\ z_{12}^{(i)} \\ z_{13}^{(i)} \\ z_{14}^{(i)} \\ z_{15}^{(i)} \end{bmatrix} \sim \mathcal{N}\left(\begin{bmatrix}0\\ 0 \\ 0 \\ 0 \\ 0 \end{bmatrix}, \begin{bmatrix}1 & 0 & 0 & 0 & 0\\ 0  & 1 & 0 & 0 & 0 \\ 0 & 0 & 1 & 0 & 0 \\ 0 & 0 & 0 & 1 & 0 \\ 0 & 0 & 0 & 0 & 1 \end{bmatrix}\right)$$\\

$$\begin{bmatrix}z_{21}^{(i)} \\ z_{22}^{(i)} \\ z_{23}^{(i)} \\ z_{24}^{(i)} \\ z_{25}^{(i)} \end{bmatrix} \sim \mathcal{N}\left(\begin{bmatrix}5\\ 5 \\ 5 \\ 5 \\ 5 \end{bmatrix}, \begin{bmatrix}1 & 0 & 0 & 0 & 0\\ 0  & 1 & 0 & 0 & 0 \\ 0 & 0 & 1 & 0 & 0 \\ 0 & 0 & 0 & 1 & 0 \\ 0 & 0 & 0 & 0 & 1 \end{bmatrix}\right)$$\\

$$\begin{bmatrix}z_{31}^{(i)} \\ z_{32}^{(i)} \\ z_{33}^{(i)} \\ z_{34}^{(i)} \\ z_{35}^{(i)} \end{bmatrix} \sim \mathcal{N}\left(\begin{bmatrix}-5\\ 5 \\ 3 \\ 4.5 \\ -6 \end{bmatrix}, 0.5\begin{bmatrix}1 & 0 & 0 & 0 & 0\\ 0  & 1 & 0 & 0 & 0 \\ 0 & 0 & 1 & 0 & 0 \\ 0 & 0 & 0 & 1 & 0 \\ 0 & 0 & 0 & 0 & 1 \end{bmatrix}\right)$$\\

$$\begin{bmatrix}z_{41}^{(i)} \\ z_{42}^{(i)} \\ z_{43}^{(i)} \\ z_{44}^{(i)} \\ z_{45}^{(i)} \end{bmatrix} \sim \mathcal{N}\left(\begin{bmatrix}5\\ -1.5 \\ -6.5 \\ 3 \\ 1 \end{bmatrix}, 0.95\begin{bmatrix}1 & 0 & 0 & 0 & 0\\ 0  & 1 & 0 & 0 & 0 \\ 0 & 0 & 1 & 0 & 0 \\ 0 & 0 & 0 & 1 & 0 \\ 0 & 0 & 0 & 0 & 1 \end{bmatrix}\right)$$\\

$$\begin{bmatrix}z_{51}^{(i)} \\ z_{52}^{(i)} \\ z_{53}^{(i)} \\ z_{54}^{(i)} \\ z_{55}^{(i)} \end{bmatrix} \sim \mathcal{N}\left(\begin{bmatrix}-2 \\ -7 \\ 9 \\ -1.5 \\ -4.5 \end{bmatrix}, 0.5\begin{bmatrix}1 & 0 & 0 & 0 & 0\\ 0  & 1 & 0 & 0 & 0 \\ 0 & 0 & 1 & 0 & 0 \\ 0 & 0 & 0 & 1 & 0 \\ 0 & 0 & 0 & 0 & 1 \end{bmatrix}\right)$$\\

$$\begin{bmatrix}z_{61}^{(i)} \\ z_{62}^{(i)} \\ z_{63}^{(i)} \\ z_{64}^{(i)} \\ z_{65}^{(i)} \end{bmatrix} \sim \mathcal{N}\left(\begin{bmatrix}-4 \\ -3 \\ -5.5 \\ 2 \\ 4 \end{bmatrix}, 0.75\begin{bmatrix}1 & 0 & 0 & 0 & 0\\ 0  & 1 & 0 & 0 & 0 \\ 0 & 0 & 1 & 0 & 0 \\ 0 & 0 & 0 & 1 & 0 \\ 0 & 0 & 0 & 0 & 1 \end{bmatrix}\right)$$\\

We have generated $15$k training examples and $10$k test examples for both of the synthetic datasets.

\subsection{Real Datasets}
The MNIST \cite{MNIST} database of gray scale handwritten digits consists of $60000$ training examples and $10000$ test samples. The CIFAR-10 \cite{krizhevsky2009learning} dataset consists of $60000$ tiny RGB images from $10$ classes, with $6000$ images per class. The standard split of this dataset consists of $50000$ training images and $10000$ test images. For experiments with MNIST and CIFAR-10, we use datasets as provided by Tensorflow API. CelebFaces Attributes Dataset (CelebA) \cite{liu2015faceattributes} is a large-scale face attributes dataset with $202599$ celebrity images, each with $40$ attribute annotations. For experiments with CELEBA, we resize the images to $64\times 64$ following many prior works \cite{kumar2020regularized, MaskAAE, dai2019diagnosing, ghosh2020from} in generative model. Table \ref{table:dataset_info} summarizes the important information about the real datasets used in this paper. Although, the test split of CELEBA dataset contains more than $10$k examples, we use $10$k randomly selected samples for FID and precision/recall score computation for all the datasets.

\begin{table}[htbp!]
  \caption{Details of Real Datasets}
  \label{table:dataset_info}
  \centering
  \resizebox{\columnwidth}{!}{
  \begin{tabular}{cccc}
    \hline
    & Dimension $(h\times w\times c)$ & Train Split Size & Test Split Size \\
    \hline
    MNIST \cite{MNIST} & $28\times28\times1$ & $60000$ & $10000$ \\
    CIFAR-$10$ \cite{krizhevsky2009learning} & $32\times32\times3$ & $50000$ & $10000$ \\    
    CELEBA \cite{liu2015faceattributes} & $64\times64\times3$ & $162770$ & $19962$ \\
    \hline
  \end{tabular}
  }
\end{table}

\section{Network Architectures}
Like any other AE based generative model, FlexAE has a reconstruction pipeline consisting of an encoder $(E_\phi)$ and a decoder $(D_\theta)$ network. We have introduced a P-GEN network consisting of a generator network $(G_\psi)$ and a critic network $(C_\kappa)$ to facilitate sampling from the latent space of the reconstruction pipeline. The generation pipeline involves the latent generator, $G_\psi$ and the image generator, $D_\theta$, meaning generation is a two-step process. First, we sample from the latent space using the latent generator, $G_\psi$. Next, the image generator, $D_\theta$ samples from the image space using the generated latent code. \par
Next, we describe the architectures of each of the components of FlexAE used for the synthetic and the real experiments.

\subsection{Synthetic Experiments}
Table \ref{table:synth_expt_arch} presents architectures of different networks used in conducting the synthetic experiment. VAE \cite{kingma2013autoencoding} consists of only encoder and decoder. WAE \cite{WAE} consists of encoder, decoder and critic. FlexAE involves all the networks.

\begin{table*}[!htbp]
  \caption{Network Architectures for Synthetic Experiment}
  \label{table:synth_expt_arch}
  \centering
  \begin{tabular}{cccc}
    \hline
    Encoder & Decoder & Generator & Critic \\
    \hline
    $\begin{aligned}[t]
        & \boldsymbol{x}\in\mathbb{R}^{128} \\
        & \to \text{FC}_{128} \to \text{ReLU} \\
        & \to \text{FC}_{m}
    \end{aligned}$ &
    
    $\begin{aligned}[t]
        & \boldsymbol{z}\in\mathbb{R}^{2} \\
        & \to \text{FC}_{128} \to \text{Tanh}
    \end{aligned}$ &

    $\begin{aligned}[t]
        & \boldsymbol{n}\in\mathbb{R}^{2} \\
        & \to \text{FC}_{128} \to \text{ReLU} \\
        & \to \text{FC}_{m}
    \end{aligned}$ &

    $\begin{aligned}[t]
        & \boldsymbol{z}\in\mathbb{R}^{2} \\
        & \to \text{FC}_{128} \to \text{ReLU} \\
        & \to \text{FC}_{128} \to \text{ReLU} \\
        & \to \text{FC}_{1}
    \end{aligned}$ \\
    \hline
    \multicolumn{4}{c}{$m=50$ for the first synthetic experiment (Figure 2 in the main paper)} \\ \multicolumn{4}{c}{and $m=2$ for the second synthetic experiment (Figure 3 in the main paper).}\\
    \hline
  \end{tabular}
\end{table*}

\subsection{Real Experiments}
For real experiments, the encoder, $E_\phi$ and the decoder, $D_\theta$ architectures are adopted from prior work \cite{kumar2020regularized}. The architecture of the encoder and the decoder networks vary from one dataset to another as presented in Table \ref{table:AEarchitecture}. However, the architectures of the generator, $G_\psi$, the critic, $C_\kappa$ and the regression network are fixed across all datasets as mentioned in Table \ref{table:GenCriticarchitecture}. The capacity (no. of trainable parameters) of $G_\psi$ and $C_\kappa$ is fairly small as compared to the AE to ensure that the adversarial training does not overfit the latent space. However, if the capacity of $G_\psi$ is too small then the bias in the latent space will increase, which will ultimately lead to a strong regularization. Therefore, we choose a moderate capacity generator and critic network. To study the effect to latent space dimensionality, $m$ on the generation quality, we train different FlexAE models with varying $m$ while everything else is kept fixed.\par
Table \ref{table:gen_architecture_bv} lists the architectures of different capacity generators used in the bias-variance experiment. Please note that the number of parameters of the latent generator model increases with model number in Table \ref{table:gen_architecture_bv}. Thus, the capacity of the Model-$1$ is the least and the capacity of the Model-$6$ is the highest.
\begin{table*}[htbp!]
  \caption{Encoder and Decoder Architectures for Real Datasets}
  \label{table:AEarchitecture}
  \centering
  \resizebox{\textwidth}{!}{
  \begin{tabular}{ccc}
    \hline
    MNIST & CIFAR10 & CELEBA \\
    \hline
    & Encoder & \\
    \hline
    $\begin{aligned}[t]
        & \boldsymbol{x}\in\mathbb{R}^{28\times28} \\
        &\to \text{Conv}_{64,4,1} \to \text{BN} \to \text{ELU}  \\
        &\to \text{Conv}_{128,4,1} \to \text{BN} \to \text{ELU}  \\
        &\to \text{Conv}_{256,4,2} \to \text{BN} \to \text{ELU}  \\
        &\to \text{Conv}_{512,4,2} \to \text{BN} \to \text{ELU}  \\
        &\to \text{Conv}_{512,4,1} \to \text{BN} \to \text{ELU}  \\
        &\to \text{Flatten} \to \text{FC}_{32}
    \end{aligned}$ &
    
    $\begin{aligned}[t]
        & \boldsymbol{x}\in\mathbb{R}^{32\times32\times3} \\
        &\to \text{Conv}_{128,4,1} \to \text{BN} \to \text{ELU}  \\
        &\to \text{Conv}_{256,4,2} \to \text{BN} \to \text{ELU}  \\
        &\to \text{Conv}_{512,4,2} \to \text{BN} \to \text{ELU}  \\
        &\to \text{Conv}_{1024,4,2} \to \text{BN} \to \text{ELU}  \\
        &\to \text{Conv}_{1024,4,1} \to \text{BN} \to \text{ELU}  \\
        &\to \text{Flatten} \to \text{FC}_{128}
    \end{aligned}$ & 
    
    $\begin{aligned}[t]
        & \boldsymbol{x}\in\mathbb{R}^{64\times64\times3} \\
        &\to \text{Conv}_{128,5,1} \to \text{BN} \to \text{ELU}  \\
        &\to \text{Conv}_{256,5,2} \to \text{BN} \to \text{ELU}  \\
        &\to \text{Conv}_{512,5,2} \to \text{BN} \to \text{ELU}  \\
        &\to \text{Conv}_{1024,5,2} \to \text{BN} \to \text{ELU}  \\
        &\to \text{Conv}_{1024,5,2} \to \text{BN} \to \text{ELU}  \\
        &\to \text{Flatten} \to \text{FC}_{128}
    \end{aligned}$ \\

    \hline
    
    & Decoder & \\
    
    \hline
    
    $\begin{aligned}[t]
        & \boldsymbol{z}\in\mathbb{R}^{32} \\
        &\to \text{FC}_{7 \times 7 \times 256} \to \text{BN} \to \text{ELU} \\
        &\to \text{Reshape}_{7 \times 7 \times 256} \\
        &\to \text{TCONV}_{512,4,1} \to \text{BN} \to \text{ELU} \\
        &\to \text{TCONV}_{256,4,1} \to \text{BN} \to \text{ELU} \\
        &\to \text{TCONV}_{128,4,2} \to \text{BN} \to \text{ELU} \\
        &\to \text{TCONV}_{64,4,2} \to \text{BN} \to \text{ELU} \\
        &\to \text{CONV}_{1,4,1} \to \text{Sigmoid}
    \end{aligned}$ &
    
    $\begin{aligned}[t]
        & \boldsymbol{z}\in\mathbb{R}^{128} \\
        &\to \text{FC}_{8 \times 8 \times 512} \to \text{BN} \to \text{ELU} \\
        &\to \text{Reshape}_{8 \times 8 \times 512} \\
        &\to \text{TCONV}_{1024,4,1} \to \text{BN} \to \text{ELU} \\
        &\to \text{TCONV}_{512,4,1} \to \text{BN} \to \text{ELU} \\
        &\to \text{TCONV}_{256,4,2} \to \text{BN} \to \text{ELU} \\
        &\to \text{TCONV}_{128,4,2} \to \text{BN} \to \text{ELU} \\
        &\to \text{CONV}_{3,4,1} \to \text{Sigmoid}
    \end{aligned}$ &
    
    $\begin{aligned}[t]
        & \boldsymbol{z}\in\mathbb{R}^{128} \\
        &\to \text{FC}_{16 \times 16 \times 512} \to \text{BN} \to \text{ELU} \\
        &\to \text{Reshape}_{16 \times 16 \times 512} \\
        &\to \text{TCONV}_{1024,5,1} \to \text{BN} \to \text{ELU} \\
        &\to \text{TCONV}_{512,5,1} \to \text{BN} \to \text{ELU} \\
        &\to \text{TCONV}_{256,5,2} \to \text{BN} \to \text{ELU} \\
        &\to \text{TCONV}_{128,5,2} \to \text{BN} \to \text{ELU} \\
        &\to \text{CONV}_{3,5,1} \to \text{Sigmoid}
    \end{aligned}$ \\
    
    \hline
  \end{tabular}
  }
\end{table*}

\begin{table*}[htbp!]
  \caption{Generator and Critic Architectures for Real Datasets}
  \label{table:GenCriticarchitecture}
  \centering
  \begin{tabular}{cc}
    \hline
    Generator & Critic \\
    \hline
    $\begin{aligned}[t]
        & \boldsymbol{n}\in\mathbb{R}^{m} \\
        & \to \text{FC}_{1024} \to \text{ReLU} \\
        & \to \text{FC}_{512} \to \text{ReLU} \\
        & \to \text{FC}_{m} \\
    \end{aligned}$ &
    
    $\begin{aligned}[t]
        & \boldsymbol{z}\in\mathbb{R}^{m} \\
        & \to \text{FC}_{512} \to \text{ReLU} \\
        & \to \text{FC}_{256} \to \text{ReLU} \\
        & \to \text{FC}_{128} \to \text{ReLU} \\
        & \to \text{FC}_{128} \to \text{ReLU} \\
        & \to \text{FC}_{1} \\
    \end{aligned}$ \\
    \hline
    \multicolumn{2}{c}{$m=32$ for MNIST and $m=128$ for CIFAR10, CELEBA.}\\
    \hline
  \end{tabular}
\end{table*}

\begin{table*}[htbp!]
  \caption{Generator Architectures for Bias-Variance Experiment}
  \label{table:gen_architecture_bv}
  \centering
  \resizebox{\textwidth}{!}{
  \begin{tabular}{cccccc}
    \hline
    Model 1 & Model 2 & Model 3 & Model 4 & Model 5 & Model 6 \\
    \hline
    $\begin{aligned}[t]
        & \boldsymbol{n}\in\mathbb{R}^{m} \\
        & \to \text{FC}_{16} \to \text{ReLU} \\
        & \to \text{FC}_{16} \to \text{ReLU} \\
        & \to \text{FC}_{m}
    \end{aligned}$ &
    $\begin{aligned}[t]
        & \boldsymbol{n}\in\mathbb{R}^{m} \\
        & \to \text{FC}_{64} \to \text{ReLU} \\
        & \to \text{FC}_{64} \to \text{ReLU} \\
        & \to \text{FC}_{m}
    \end{aligned}$ &
    $\begin{aligned}[t]
        & \boldsymbol{n}\in\mathbb{R}^{m} \\
        & \to \text{FC}_{256} \to \text{ReLU} \\
        & \to \text{FC}_{256} \to \text{ReLU} \\
        & \to \text{FC}_{m}
    \end{aligned}$ &
    $\begin{aligned}[t]
        & \boldsymbol{n}\in\mathbb{R}^{m} \\
        & \to \text{FC}_{1024} \to \text{ReLU} \\
        & \to \text{FC}_{1024} \to \text{ReLU} \\
        & \to \text{FC}_{m}
    \end{aligned}$ &
    $\begin{aligned}[t]
        & \boldsymbol{n}\in\mathbb{R}^{m} \\
        & \to \text{FC}_{1024} \to \text{ReLU} \\
        & \to \text{FC}_{1024} \to \text{ReLU} \\
        & \to \text{FC}_{1024} \to \text{ReLU} \\
        & \to \text{FC}_{1024} \to \text{ReLU} \\
        & \to \text{FC}_{m}
    \end{aligned}$ &
    $\begin{aligned}[t]
        & \boldsymbol{n}\in\mathbb{R}^{m} \\
        & \to \text{FC}_{2048} \to \text{ReLU} \\
        & \to \text{FC}_{2048} \to \text{ReLU} \\
        & \to \text{FC}_{2048} \to \text{ReLU} \\
        & \to \text{FC}_{2048} \to \text{ReLU} \\
        & \to \text{FC}_{m}
    \end{aligned}$ \\
    \hline
    \multicolumn{6}{c}{$m=32$ for MNIST and $m=128$ for CIFAR10, CELEBA.}\\
    \hline
  \end{tabular}
  }
\end{table*}

\begin{figure*}[ht!]
    \centering
    \begin{subfigure}[t]{0.24\textwidth}
        \centering
        \includegraphics[trim={2 2 2 2}, clip, keepaspectratio,width=\linewidth]{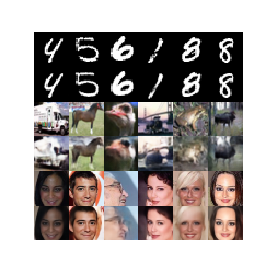}
        \caption{}
        \label{fig:recon}
    \end{subfigure}%
    ~ 
    \begin{subfigure}[t]{0.24\textwidth}
        \centering
        \includegraphics[trim={7 7 7 7}, clip, keepaspectratio, width=\linewidth]{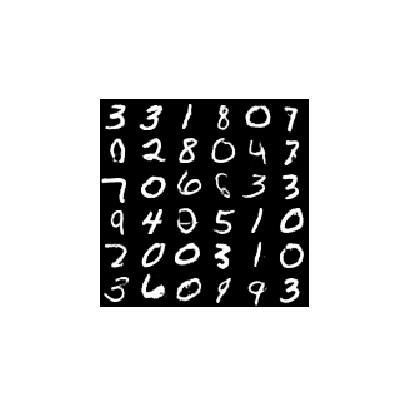}
        \caption{}
        \label{fig:mnist_gen}
    \end{subfigure}%
    ~ 
    \begin{subfigure}[t]{0.24\textwidth}
        \centering
        \includegraphics[trim={7 7 7 7}, clip, keepaspectratio, width=\linewidth]{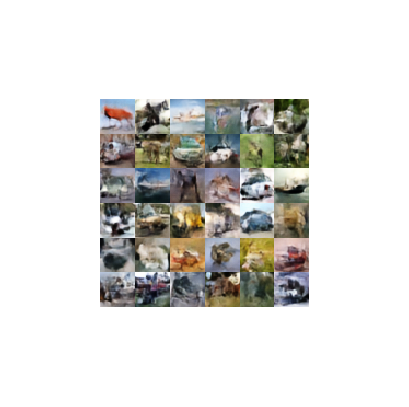}
        \caption{}
        \label{fig:cifar10_gen}
    \end{subfigure}%
    ~ 
    \begin{subfigure}[t]{0.24\textwidth}
        \centering
        \includegraphics[trim={7 7 7 7}, clip, keepaspectratio, width=\linewidth]{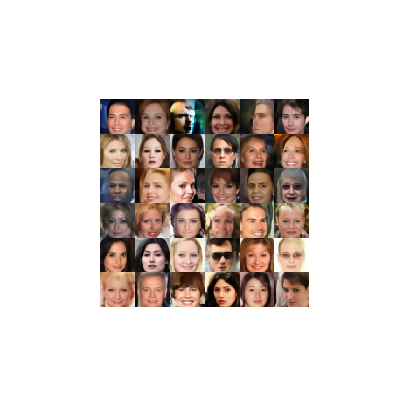}
        \caption{}
        \label{fig:celeba_gen}
    \end{subfigure}
    \caption{(a) Visualization of reconstruction quality of FlexAE model on randomly selected data from the test split of MNIST (first and second rows), CIFAR-$10$ (third and fourth rows) and CELEBA (fifth and sixth rows). The odd rows represent the real data and the even rows represent reconstructed data. Randomly generated samples from (b) MNIST, (c) CIFAR-$10$, and (d) CELEBA datasets using FlexAE model.}
    \label{fig:recon_gen}
\end{figure*}

\section{Training Algorithm, Hyper-parameters, Computing Resources and Average Runtime}
As mentioned in the main paper, the auto-encoder is required to be optimized jointly with the P-GEN to ensure regularization in the AE latent space. This regularization effectively enforces smoothness in the learnt latent space and prevents the AE from  overfitting on the training examples. In order to be able to satisfy the above requirement in practice, we optimize each of the four losses specified in the main paper in every training iteration. Specifically, in each learning loop, we optimize the $L_{AE}$, $L_{Critic}$, $L_{Gen}$, and $L_{Enc}$ in that order using a learning schedule. We use Adam optimizer for our optimization. The training algorithm is described in Algorithm \ref{alg:flexae-training-loop}. For real experiments we have trained our models for $130000$ iterations on each dataset with a batch size of $128$. We have used a machine with Intel\textsuperscript{\textregistered} Xeon\textsuperscript{\textregistered} Gold 6142 CPU, 376GiB RAM, and Zotac GeForce\textsuperscript{\textregistered} GTX 1080 Ti 11GB Graphic Card for all of our experiments. The average runtime for experiments on MNIST, CIFAR-10, and CELEBA is approximately 20 hours, 40 hours and 100 hours respectively.
\begin{algorithm}
    \caption{Pseudo code for the training loop of FlexAE \label{alg:flexae-training-loop}}
    \hspace*{\algorithmicindent} \textbf{Hyper-parameters:} $\eta_{AE} = 0.001$, $\eta_{Critic} = 0.0001$, $\eta_{Gen} = 0.0005$, $\eta_{Enc} = 0.00001$, $\text{AE\_OPT} = \text{Adam}(\text{lr}=\eta_{AE}, \beta_1=0.9, \beta_2=0.999)$, $\text{CRITIC\_OPT} = \text{Adam}(\text{lr}=\eta_{Critic}, \beta_1=0.0, \beta_2=0.9)$, $\text{GEN\_OPT} = \text{Adam}(\text{lr}=\eta_{Gen}, \beta_1=0.0, \beta_2=0.9)$, $\text{ENC\_OPT} = \text{Adam}(\text{lr}=\eta_{Enc}, \beta_1=0.0, \beta_2=0.9)$, $disc\_training\_ratio=5$. \\
    \begin{algorithmic}[1]
        \Function{Train}{}
            \For{$i \gets 1 \textrm{ to } training\_steps$}
                \State Minimize $L_{AE}$ and Update $\phi,~\theta$
                \For{$j \gets 1 \textrm{ to } disc\_training\_ratio$}
                    \State Minimize $L_{Critic}$ and Update $\kappa$
                \EndFor
                \State Minimize $L_{Gen}$ and Update $\psi$
                \State Minimize $L_{Enc}$ and Update $\phi$
            \EndFor
        \EndFunction
    \end{algorithmic}
\end{algorithm}

\section{Experimental Results}
In the main paper, the performance of FlexAE is evaluated mainly quantitatively, using standard metrics: FID \cite{NIPS2017_7240} and precision/recall \cite{Pre_rec} score. We have used $10000$ reconstructed and $10000$ generated samples against $10000$ test examples for computation of FID and precision/recall score for all datasets. It has been observed that FlexAE outperforms all other current state-of-the-art AE based generative models as measured using those metrics. In this section, we present more qualitative results (reconstruction on test examples, generated samples and resulting images due to interpolation in the latent space) for visual evaluation of the proposed generative framework, FlexAE.\par
Figure \ref{fig:recon} represents reconstruction of $6$ randomly chosen samples from test test split of MNIST (row 1 and 2), CIFAR-$10$ (row 3 and 4), and CELEBA (row 5 and 6) dataset. The odd rows represent true data and the even rows represents reconstructed data. Figure \ref{fig:mnist_gen}, \ref{fig:cifar10_gen}, \ref{fig:celeba_gen} presents $36$ randomly generated samples of MNIST, CIFAR-$10$ and CELEBA datasets respectively.\par
Next, we present more attribute based interpolation results from the CELEBA test split in Figure \ref{fig:test_big_nose}, Figure \ref{fig:test_heavy_makeup}, Figure \ref{fig:test_black_hair}, Figure \ref{fig:test_smiling}, and Figure \ref{fig:test_male} for the attributes ``Big Nose'', ``Heavy Makeup'', ``Black Hair'', ``Smiling'', and ``Male'' respectively. The central image of the grid in the sub-figures (a) and (b) in every figure presents a negative test example from the CELEBA dataset i.e. a test sample without the corresponding attribute. Whereas, the central image in the grid of the sub-figures (c) and (d) presents a positive test example i.e. a test sample with the particular attribute. For latent space traversal along a particular attribute direction, we calculate the average representation, $z_{pos}$ with respect to all the positive training samples and the average representation, $z_{neg}$ with respect to all the negative training samples. Finally, we use the direction $(z_{pos} - z_{neg})$ to traverse the latent space for attribute manipulation. Please note, this supervised traversal is performed post training in order to understand if the trained model could learn the meaning of different face attributes without supervision. The training was completely unsupervised without using any label information. As can be seen from the Figures \ref{fig:test_big_nose} - \ref{fig:test_male}, FlexAE could successfully learn the concept of different attributes without any kind of supervision. Otherwise, the interpolated figures would not be so smooth.\par
Finally, Figure \ref{fig:supp_knn_15x15} presents a $15\times15$ grid, where, the first column plots some randomly generated face images and the remaining entries in each row are the $14$ nearest neighbours (in terms of Euclidean distance) from the training split. The generated images are visually significantly different as compared to the nearest training examples. This confirms that FlexAE has not memorised the training examples and generates unique, unseen images.

% ######################## Interpolation ######################## %
\begin{figure*}[!h]
    \centering
    \begin{subfigure}[t]{0.24\textwidth}
        \centering
        \includegraphics[trim={4 4 4 4}, clip, keepaspectratio,width=\linewidth]{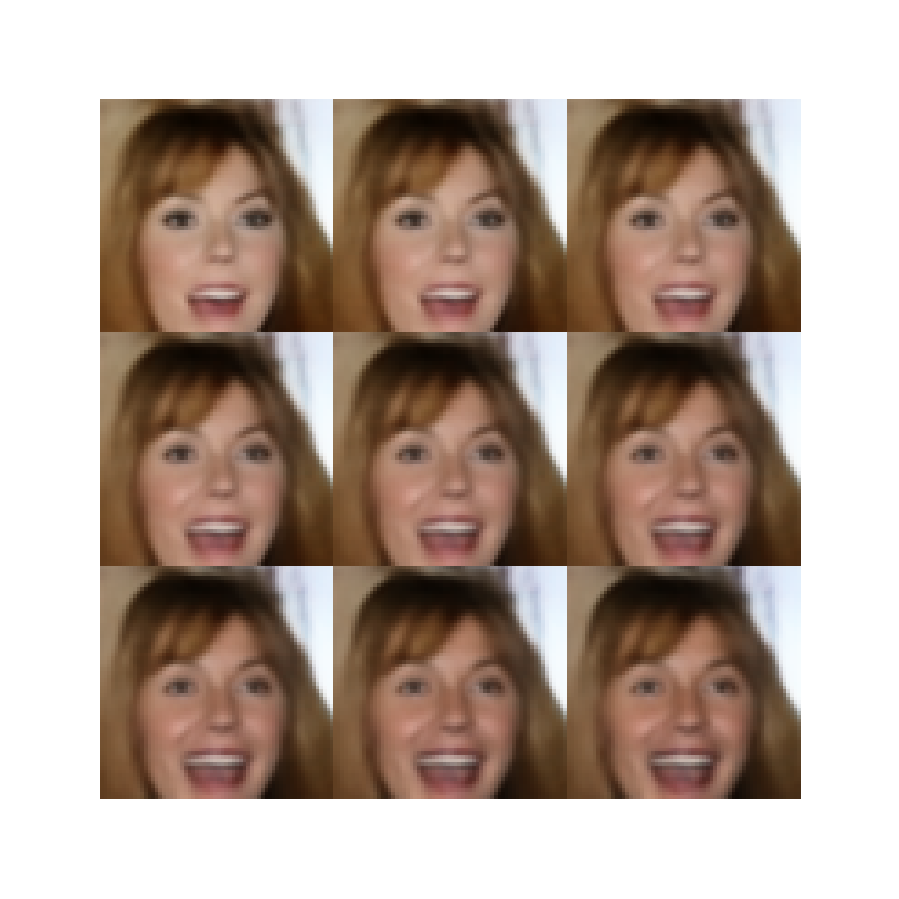}
        \caption{}
        \label{fig:big_nose_test_neg_1}
    \end{subfigure}%
    ~ 
    \begin{subfigure}[t]{0.24\textwidth}
        \centering
        \includegraphics[trim={4 4 4 4}, clip, keepaspectratio, width=\linewidth]{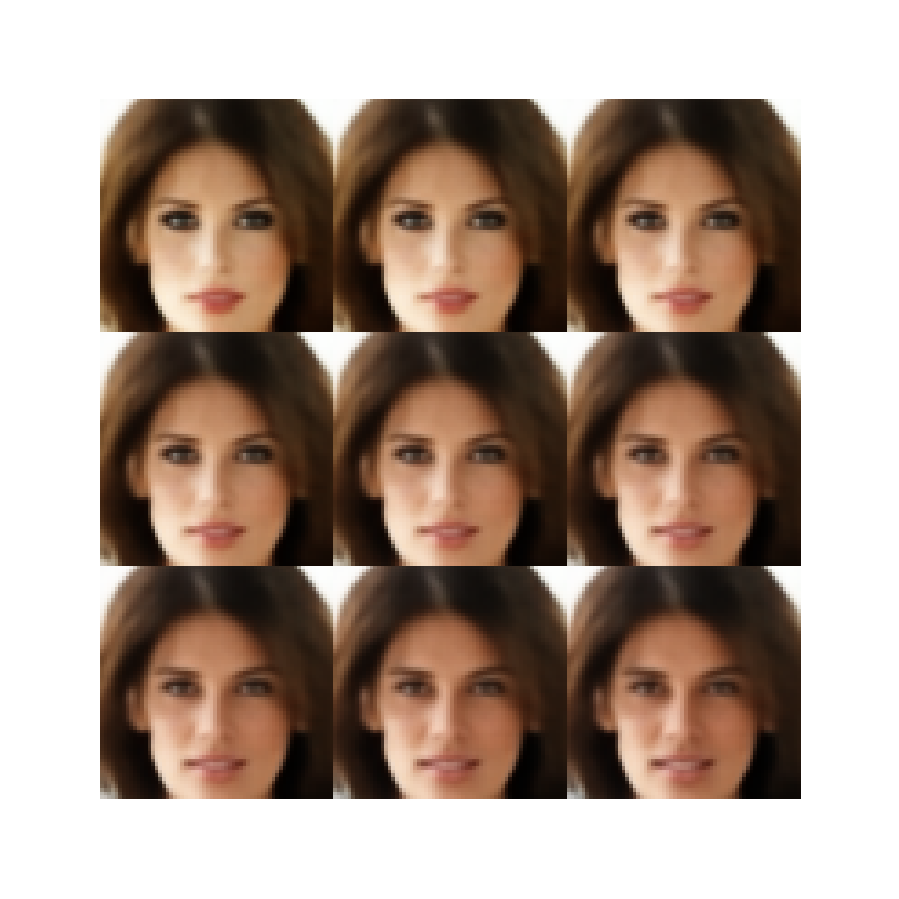}
        \caption{}
        \label{fig:big_nose_test_neg_2}
    \end{subfigure}%
    ~ 
    \begin{subfigure}[t]{0.24\textwidth}
        \centering
        \includegraphics[trim={4 4 4 4}, clip, keepaspectratio, width=\linewidth]{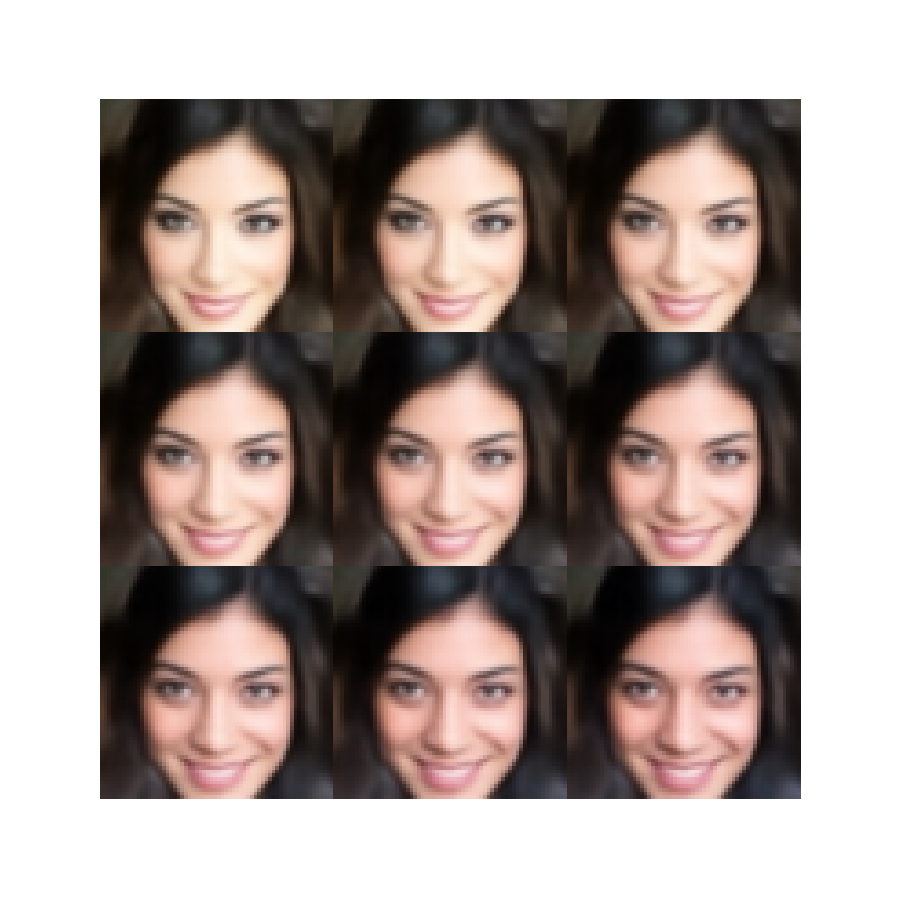}
        \caption{}
        \label{fig:big_nose_test_pos_1}
    \end{subfigure}%
    ~ 
    \begin{subfigure}[t]{0.24\textwidth}
        \centering
        \includegraphics[trim={4 4 4 4}, clip, keepaspectratio, width=\linewidth]{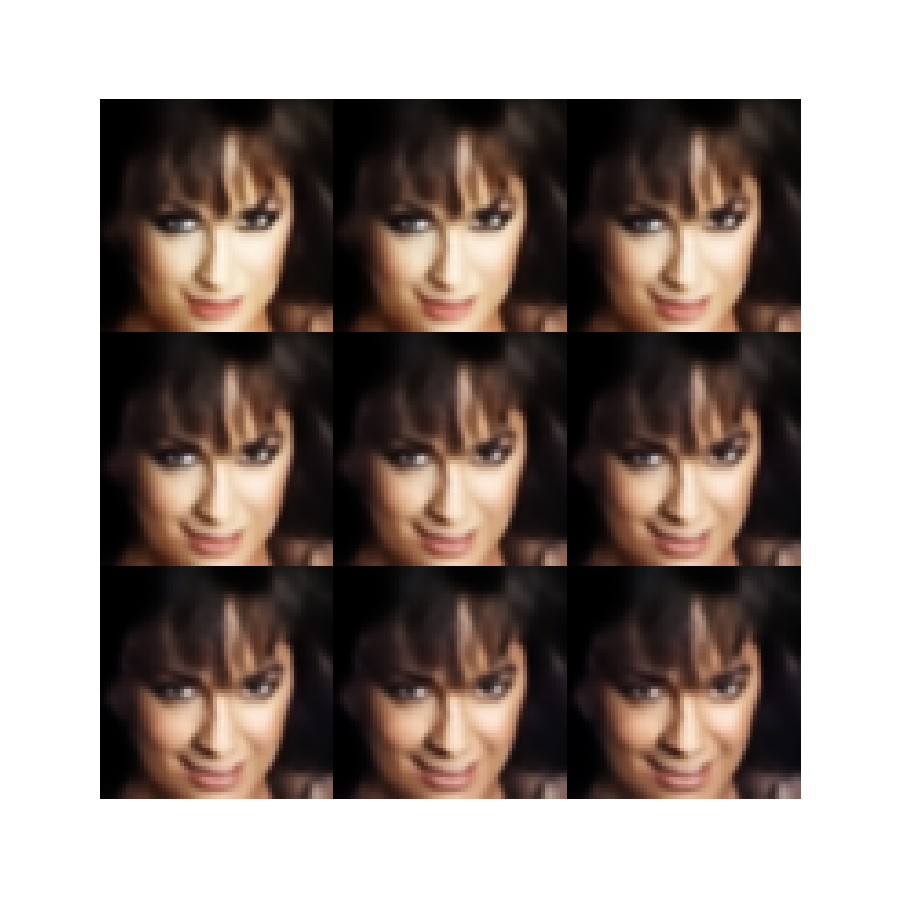}
        \caption{}
        \label{fig:big_nose_test_pos_2}
    \end{subfigure}
    \caption{Interpolations in the latent space of FlexAE on CelebA. Each row in (a) and (b) presents manipulation of the attribute ``Big Nose''. The central image of each grid in (a), and (b) is a true image from the test split without the attribute. Whereas, the central image of each grid in (c) and (d) is a true image from the test split with the attribute.}
    \label{fig:test_big_nose}
\end{figure*}

\begin{figure*}[!h]
    \centering
    \begin{subfigure}[t]{0.24\textwidth}
        \centering
        \includegraphics[trim={4 4 4 4}, clip, keepaspectratio,width=\linewidth]{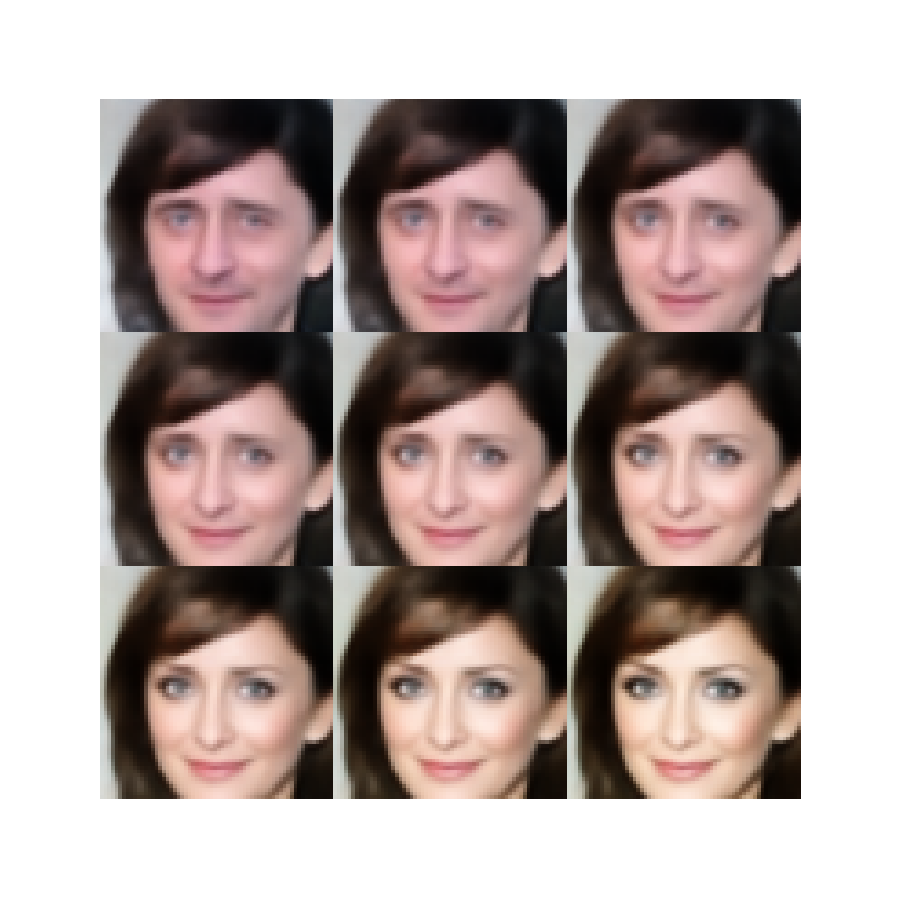}
        \caption{}
        \label{fig:heavy_makeup_test_neg_1}
    \end{subfigure}%
    ~ 
    \begin{subfigure}[t]{0.24\textwidth}
        \centering
        \includegraphics[trim={4 4 4 4}, clip, keepaspectratio, width=\linewidth]{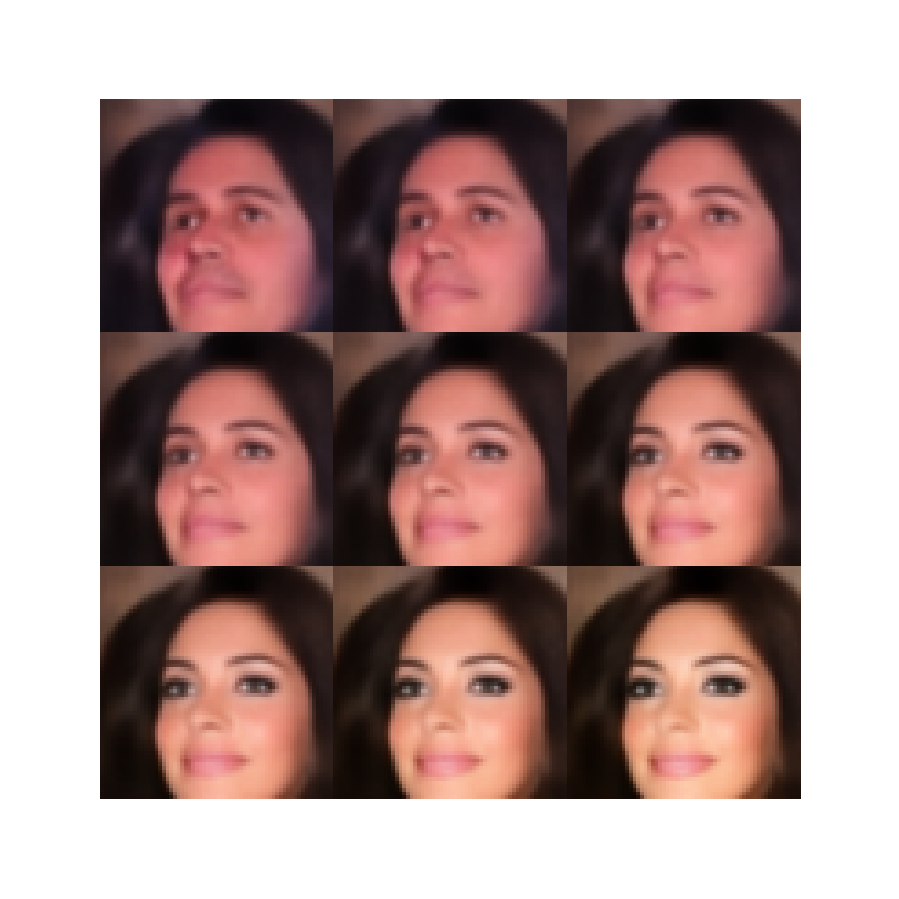}
        \caption{}
        \label{fig:heavy_makeup_test_neg_2}
    \end{subfigure}%
    ~ 
    \begin{subfigure}[t]{0.24\textwidth}
        \centering
        \includegraphics[trim={4 4 4 4}, clip, keepaspectratio, width=\linewidth]{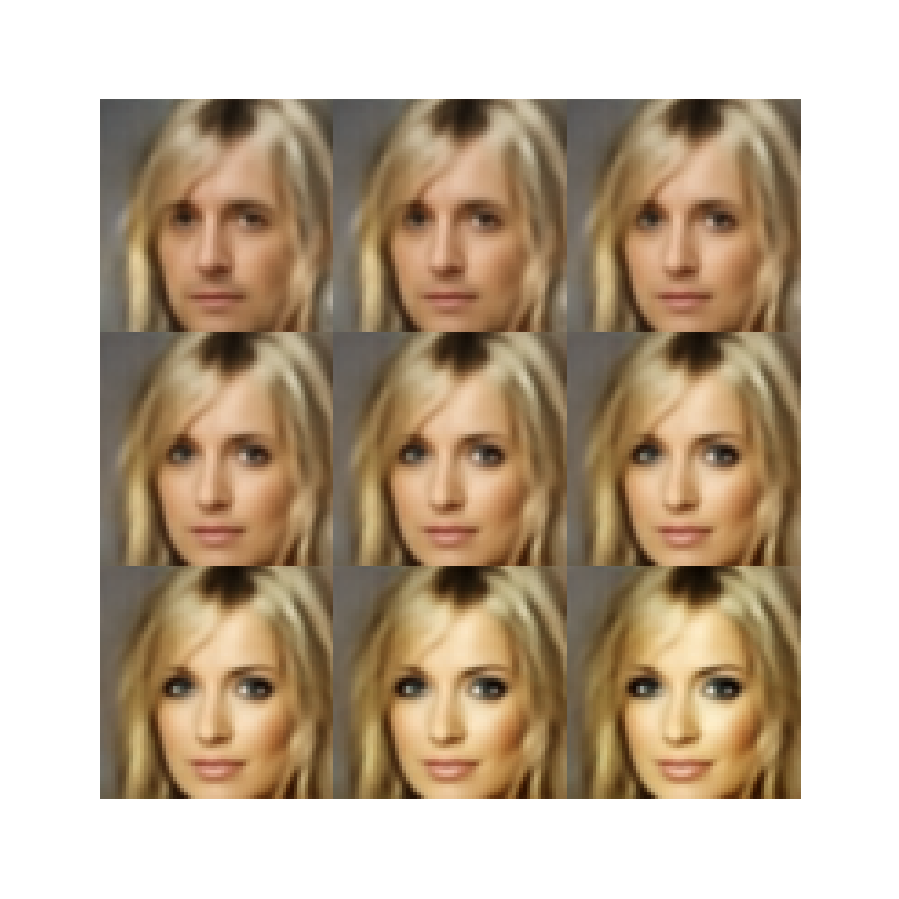}
        \caption{}
        \label{fig:heavy_makeup_test_pos_1}
    \end{subfigure}%
    ~ 
    \begin{subfigure}[t]{0.24\textwidth}
        \centering
        \includegraphics[trim={4 4 4 4}, clip, keepaspectratio, width=\linewidth]{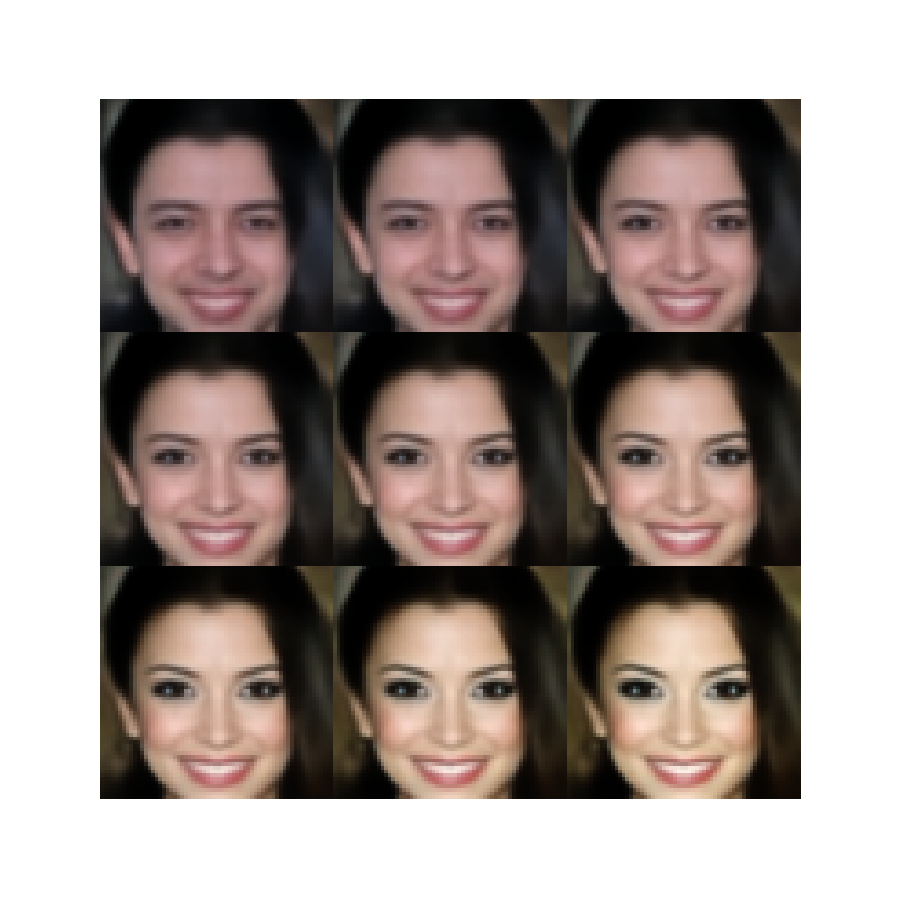}
        \caption{}
        \label{fig:heavy_makeup_test_pos_2}
    \end{subfigure}
    \caption{Interpolations in the latent space of FlexAE on CelebA. Each row in (a) and (b) presents manipulation of the attribute ``Heavy Makeup''. The central image of each grid in (a), and (b) is a true image from the test split without the attribute. Whereas, the central image of each grid in (c) and (d) is a true image from the test split with the attribute.}
    \label{fig:test_heavy_makeup}
\end{figure*}

\begin{figure*}[htbp]
    \centering
    \begin{subfigure}[t]{0.24\textwidth}
        \centering
        \includegraphics[trim={4 4 4 4}, clip, keepaspectratio,width=\linewidth]{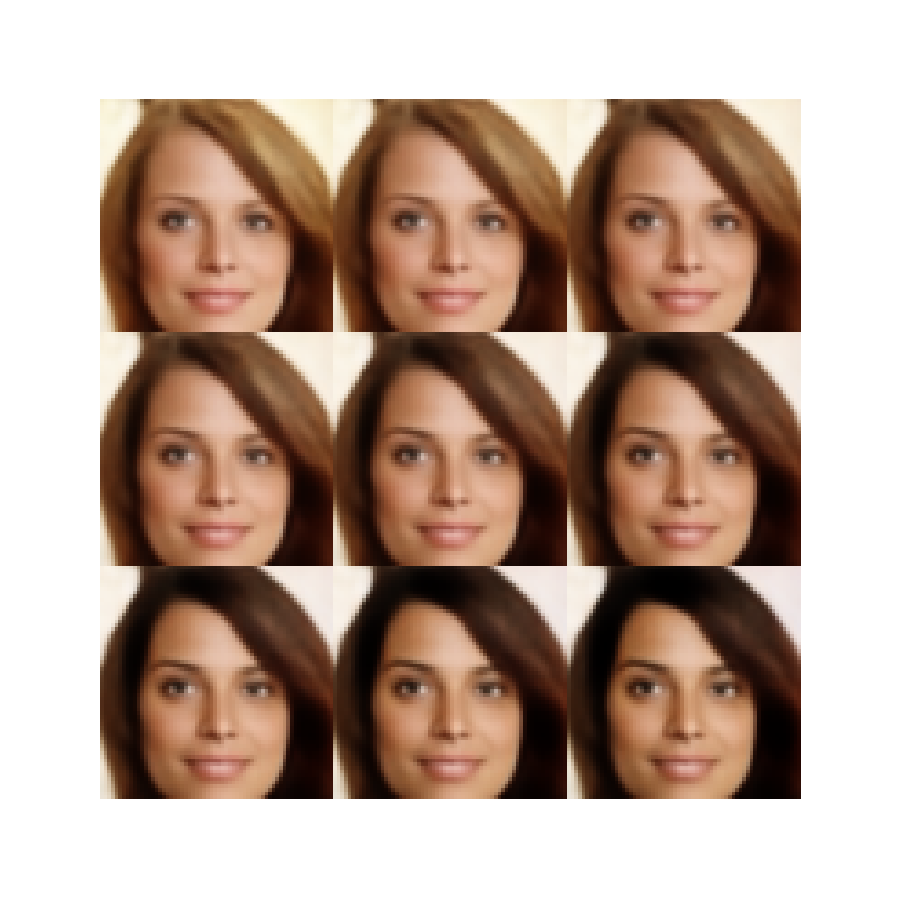}
        \caption{}
        \label{fig:black_hair_test_neg_1}
    \end{subfigure}%
    ~ 
    \begin{subfigure}[t]{0.24\textwidth}
        \centering
        \includegraphics[trim={4 4 4 4}, clip, keepaspectratio, width=\linewidth]{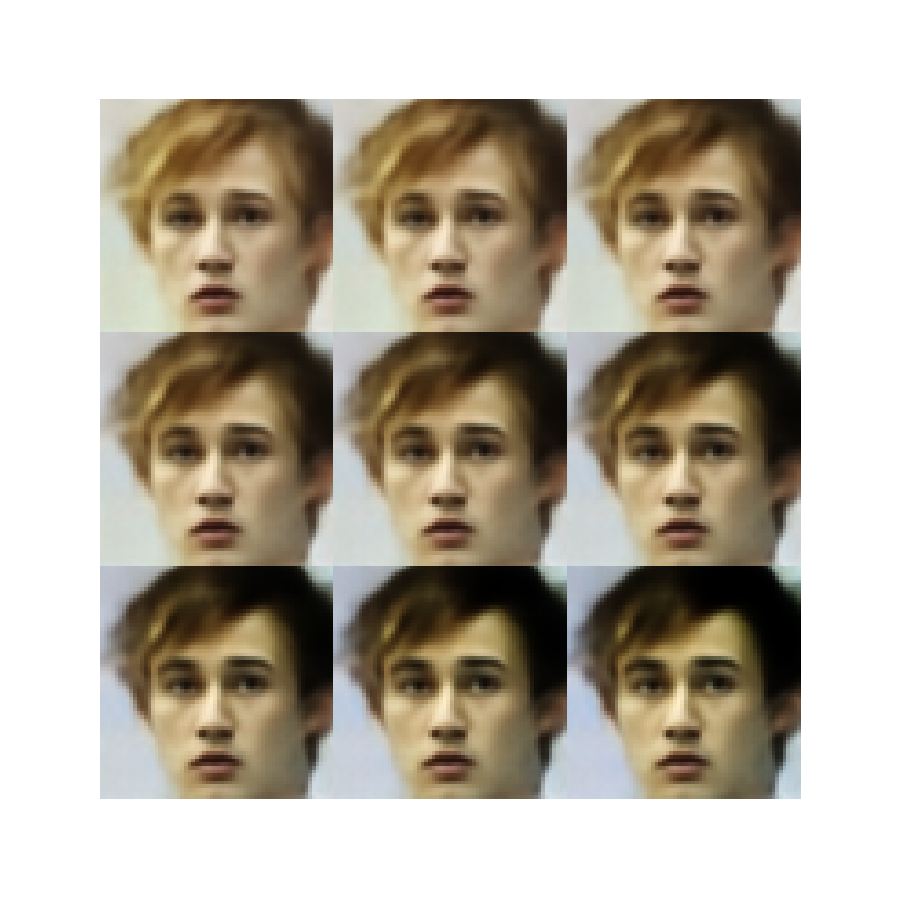}
        \caption{}
        \label{fig:black_hair_test_neg_2}
    \end{subfigure}%
    ~ 
    \begin{subfigure}[t]{0.24\textwidth}
        \centering
        \includegraphics[trim={4 4 4 4}, clip, keepaspectratio, width=\linewidth]{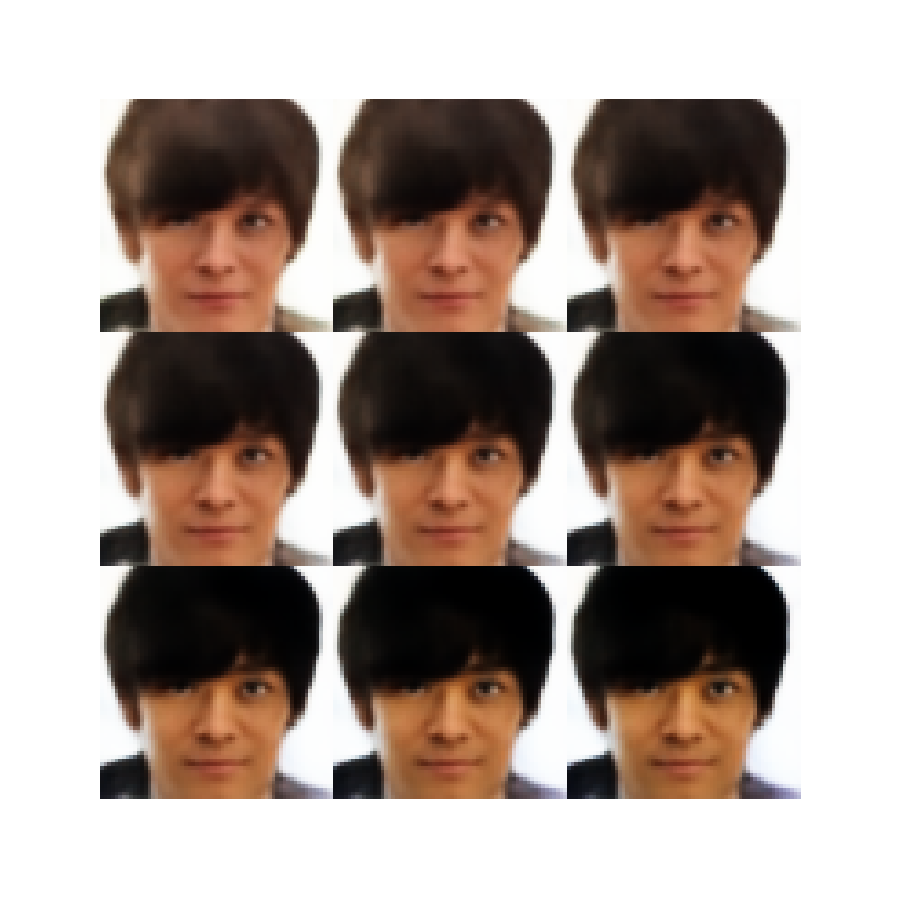}
        \caption{}
        \label{fig:black_hair_test_pos_1}
    \end{subfigure}%
    ~ 
    \begin{subfigure}[t]{0.24\textwidth}
        \centering
        \includegraphics[trim={4 4 4 4}, clip, keepaspectratio, width=\linewidth]{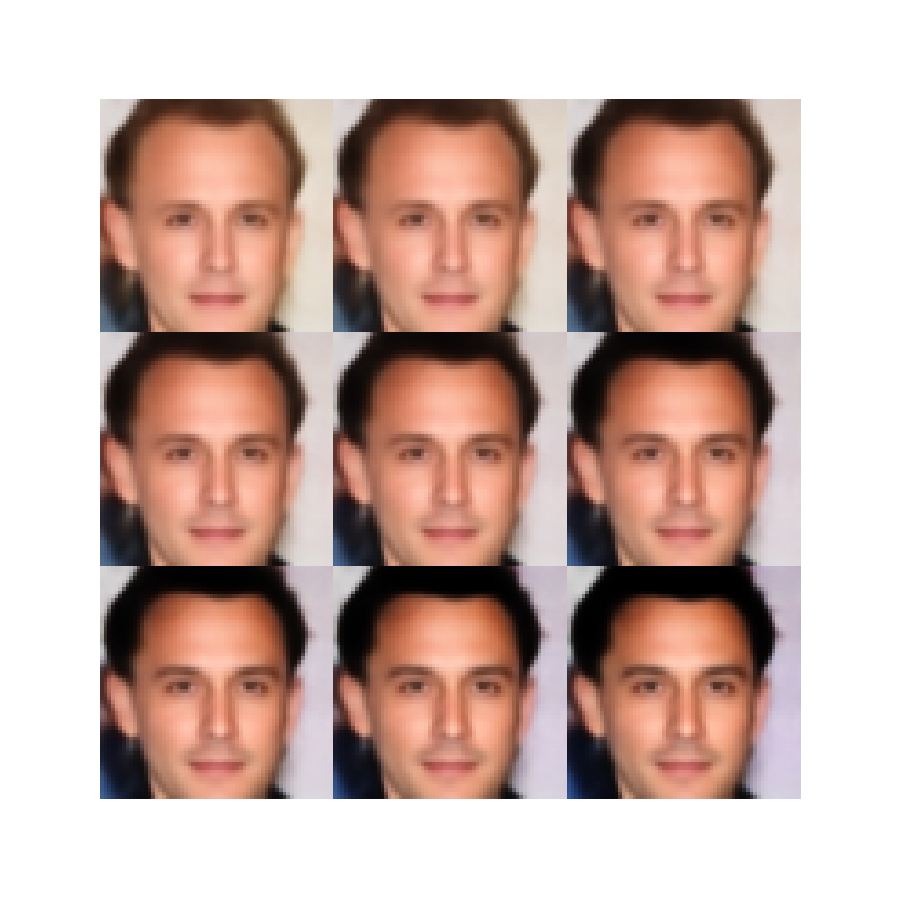}
        \caption{}
        \label{fig:black_hair_test_pos_2}
    \end{subfigure}
    \caption{Interpolations in the latent space of FlexAE on CelebA. Each row in (a) and (b) presents manipulation of the attribute ``Black Hair''. The central image of each grid in (a), and (b) is a true image from the test split without the attribute. Whereas, the central image of each grid in (c) and (d) is a true image from the test split with the attribute.}
    \label{fig:test_black_hair}
\end{figure*}

\begin{figure*}[htbp]
    \centering
    \begin{subfigure}[t]{0.24\textwidth}
        \centering
        \includegraphics[trim={4 4 4 4}, clip, keepaspectratio,width=\linewidth]{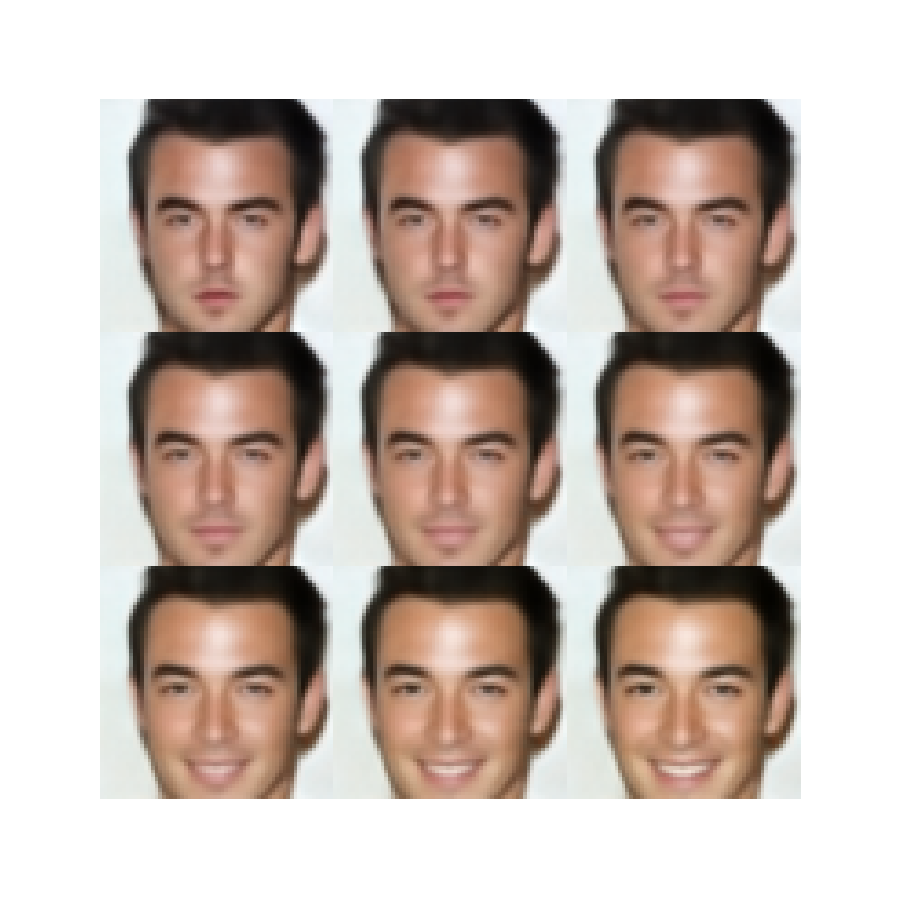}
        \caption{}
        \label{fig:smiling_test_neg_1}
    \end{subfigure}%
    ~ 
    \begin{subfigure}[t]{0.24\textwidth}
        \centering
        \includegraphics[trim={4 4 4 4}, clip, keepaspectratio, width=\linewidth]{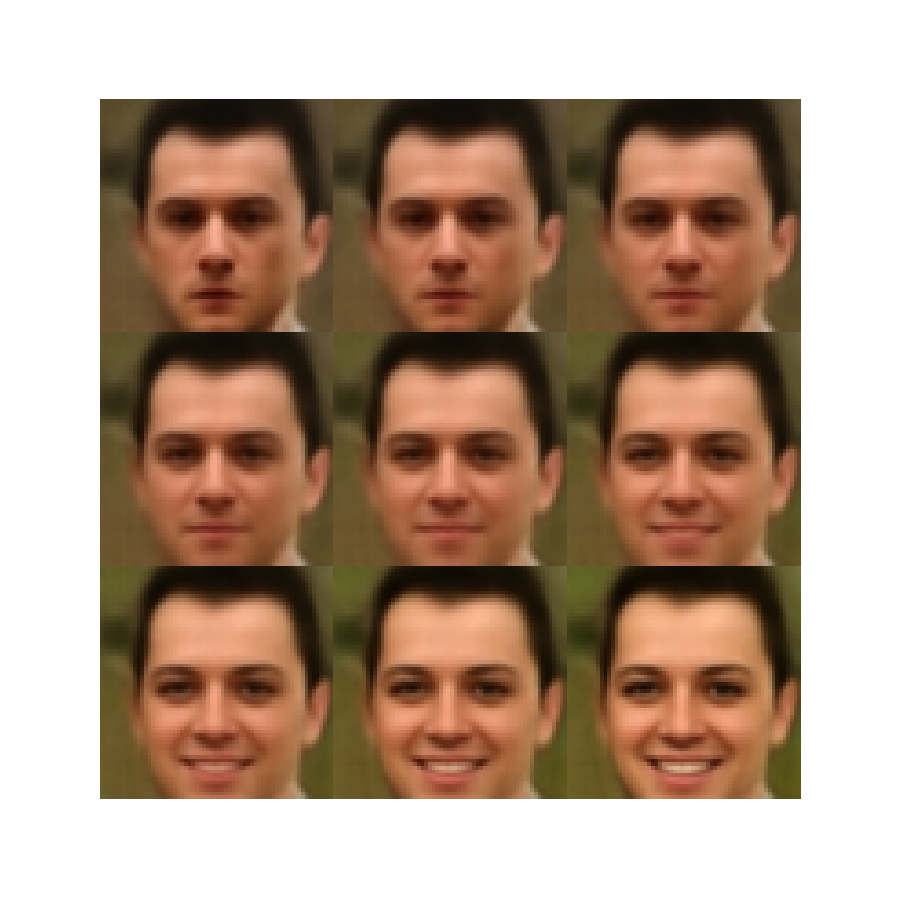}
        \caption{}
        \label{fig:smiling_test_neg_2}
    \end{subfigure}%
    ~ 
    \begin{subfigure}[t]{0.24\textwidth}
        \centering
        \includegraphics[trim={4 4 4 4}, clip, keepaspectratio, width=\linewidth]{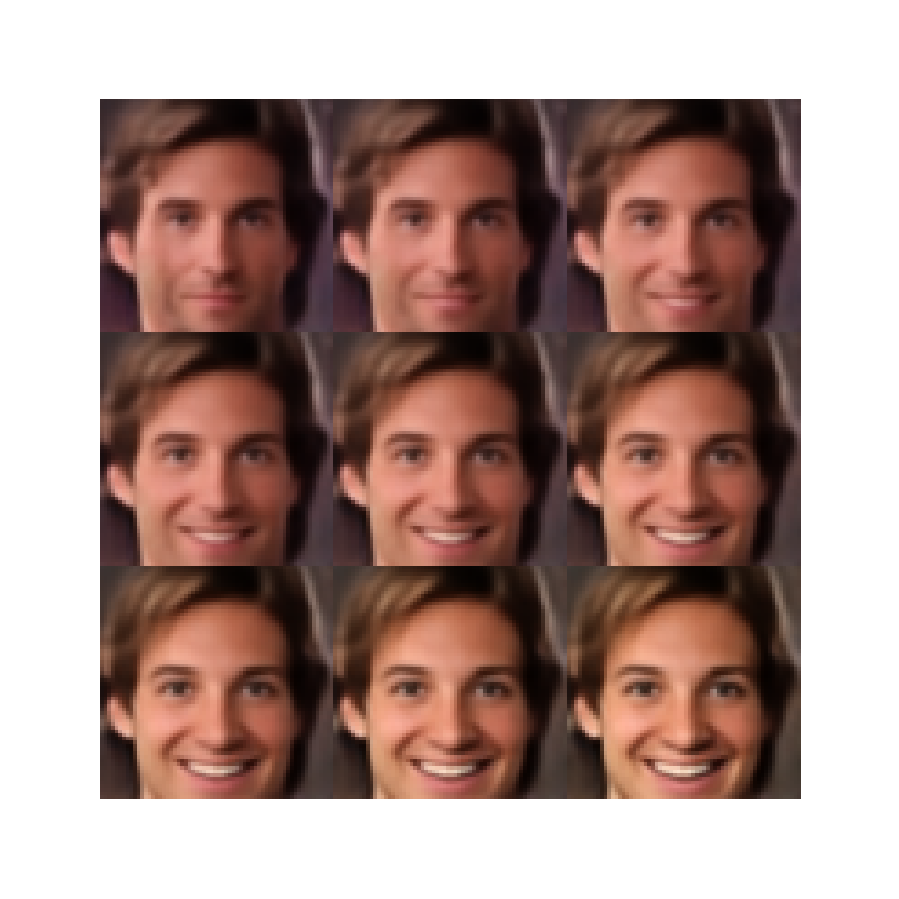}
        \caption{}
        \label{fig:smiling_test_pos_1}
    \end{subfigure}%
    ~ 
    \begin{subfigure}[t]{0.24\textwidth}
        \centering
        \includegraphics[trim={4 4 4 4}, clip, keepaspectratio, width=\linewidth]{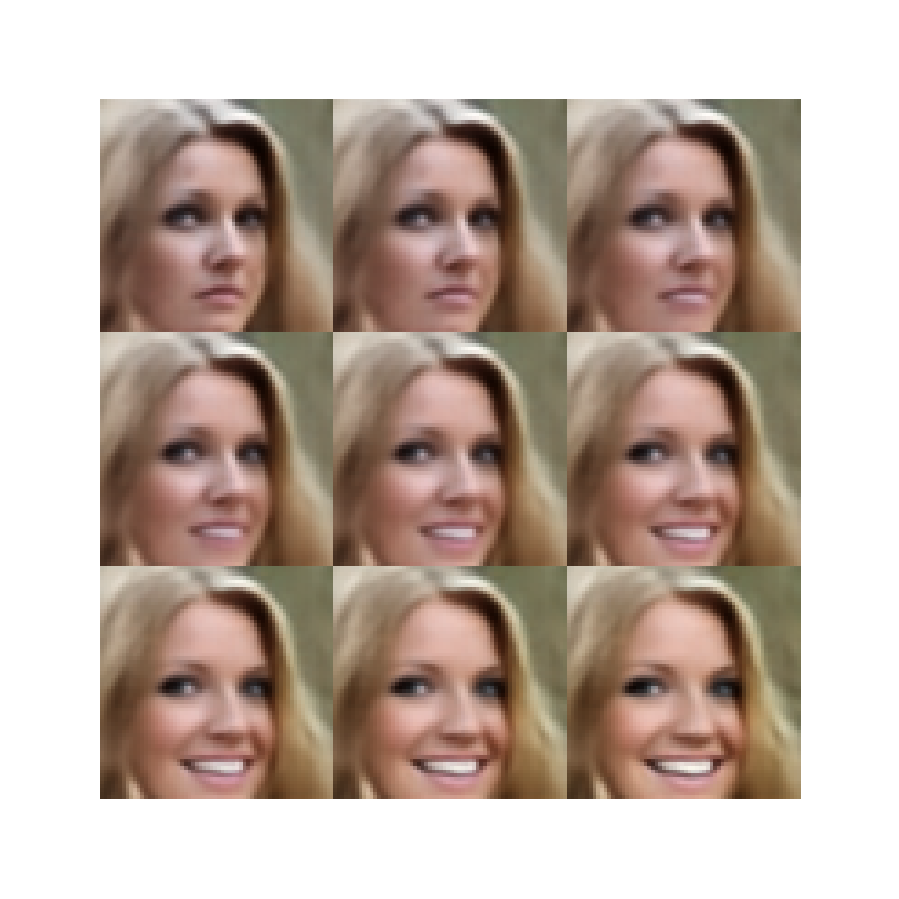}
        \caption{}
        \label{fig:smiling_test_pos_2}
    \end{subfigure}
    \caption{Interpolations in the latent space of FlexAE on CelebA. Each row in (a) and (b) presents manipulation of the attribute ``Smiling''. The central image of each grid in (a), and (b) is a true image from the test split without the attribute. Whereas, the central image of each grid in (c) and (d) is a true image from the test split with the attribute.}
    \label{fig:test_smiling}
\end{figure*}

\begin{figure*}[htbp]
    \centering
    \begin{subfigure}[t]{0.24\textwidth}
        \centering
        \includegraphics[trim={4 4 4 4}, clip, keepaspectratio,width=\linewidth]{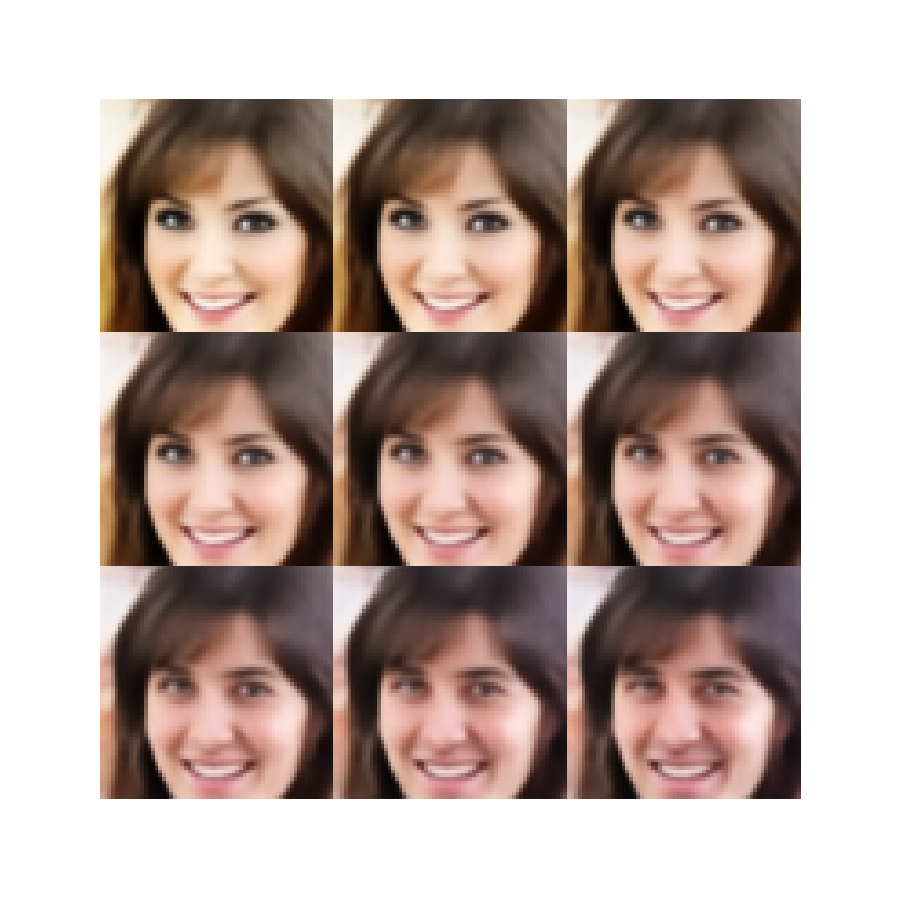}
        \caption{}
        \label{fig:male_test_neg_1}
    \end{subfigure}%
    ~ 
    \begin{subfigure}[t]{0.24\textwidth}
        \centering
        \includegraphics[trim={4 4 4 4}, clip, keepaspectratio, width=\linewidth]{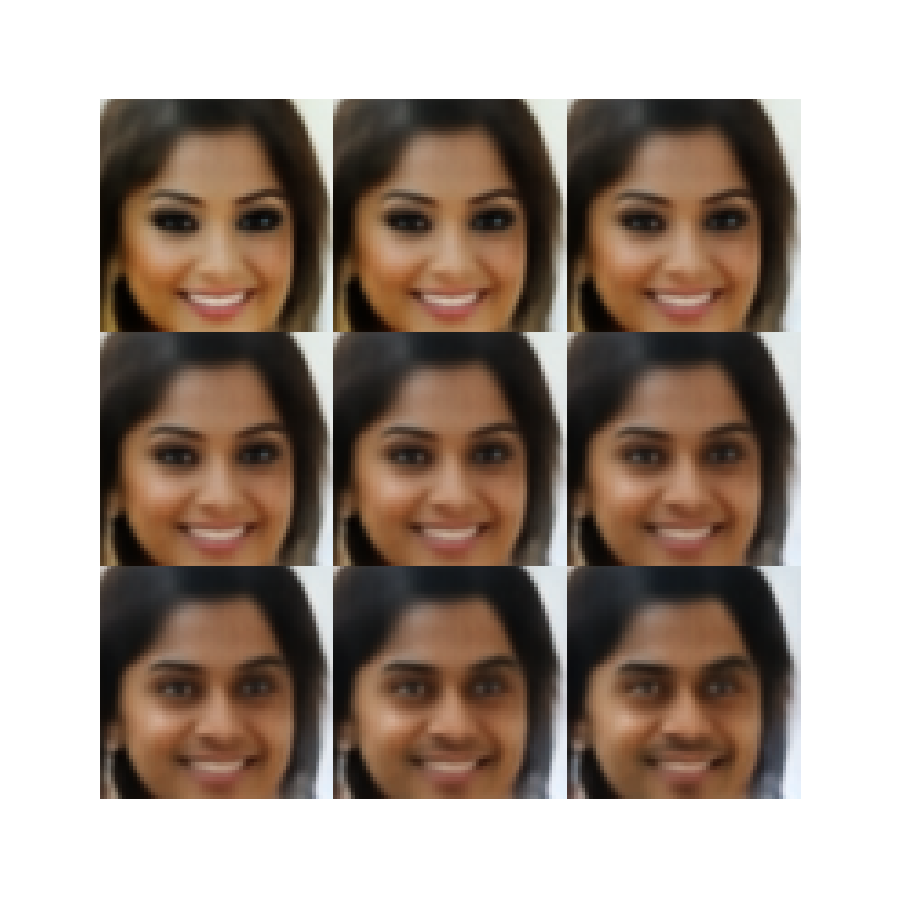}
        \caption{}
        \label{fig:male_test_neg_2}
    \end{subfigure}%
    ~ 
    \begin{subfigure}[t]{0.24\textwidth}
        \centering
        \includegraphics[trim={4 4 4 4}, clip, keepaspectratio, width=\linewidth]{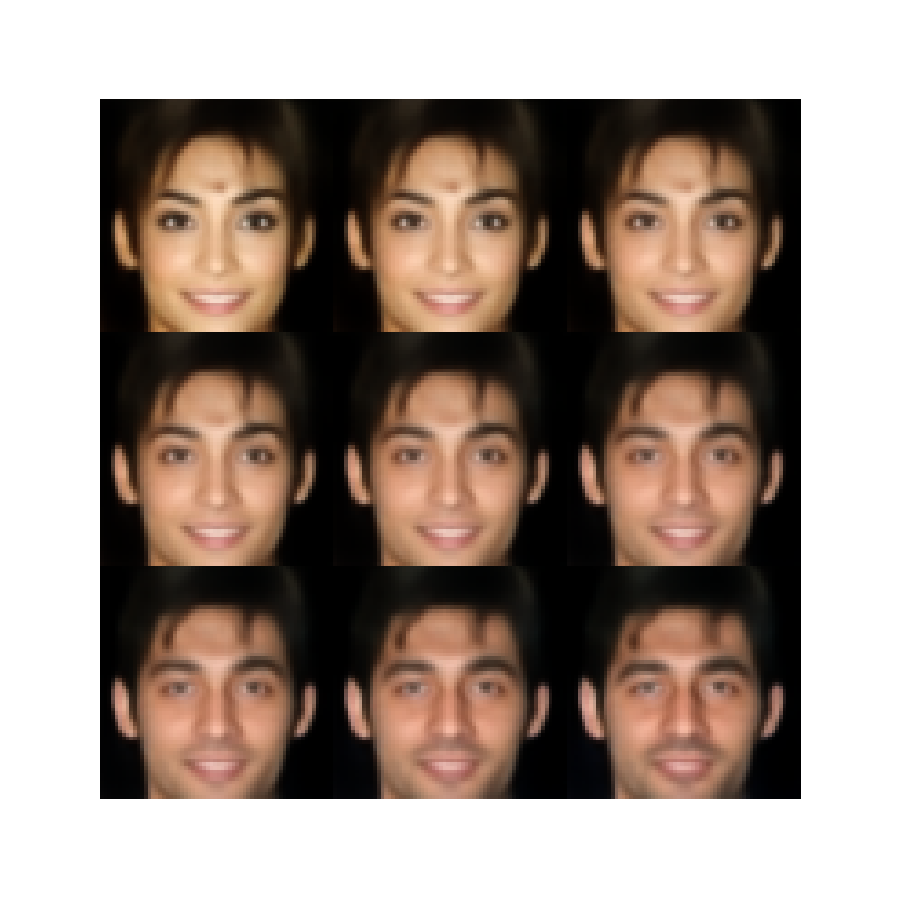}
        \caption{}
        \label{fig:male_test_pos_1}
    \end{subfigure}%
    ~ 
    \begin{subfigure}[t]{0.24\textwidth}
        \centering
        \includegraphics[trim={4 4 4 4}, clip, keepaspectratio, width=\linewidth]{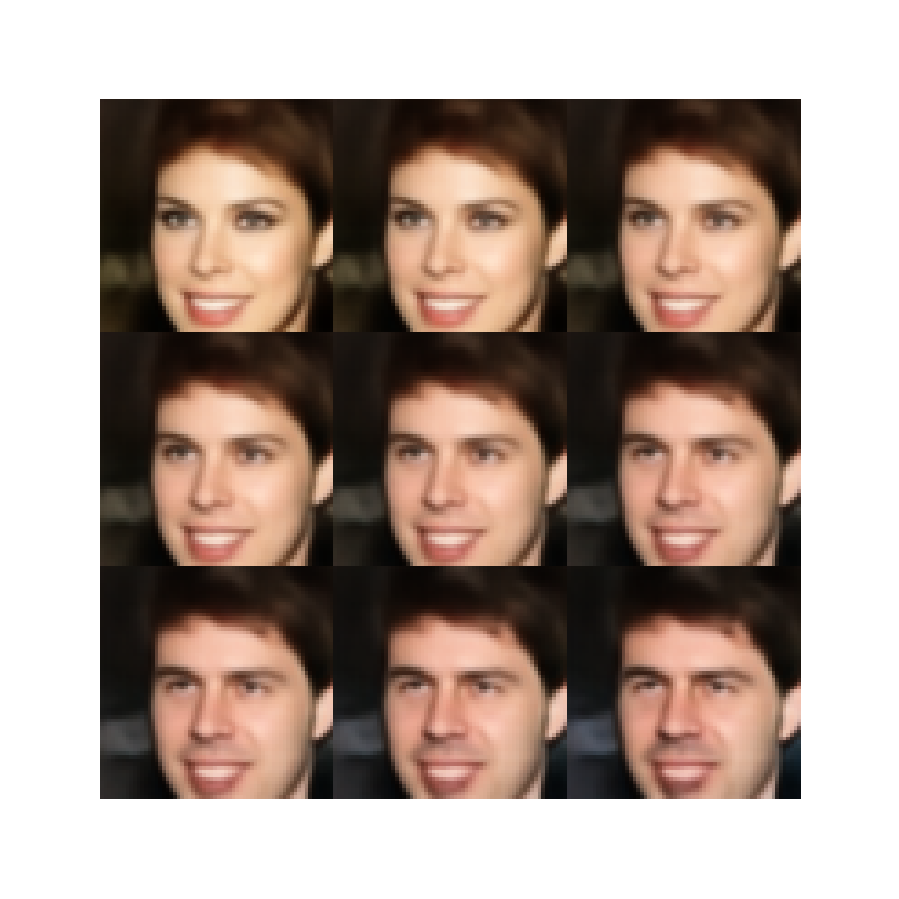}
        \caption{}
        \label{fig:male_test_pos_2}
    \end{subfigure}
    \caption{Interpolations in the latent space of FlexAE on CelebA. Each row in (a) and (b) presents manipulation of the attribute ``Male''. The central image of each grid in (a), and (b) is a true image from the test split without the attribute. Whereas, the central image of each grid in (c) and (d) is a true image from the test split with the attribute.}
    \label{fig:test_male}
\end{figure*}

\begin{figure*}[htbp]
    \centering
    \includegraphics[trim={5 5 5 5}, clip, keepaspectratio, width=\textwidth]{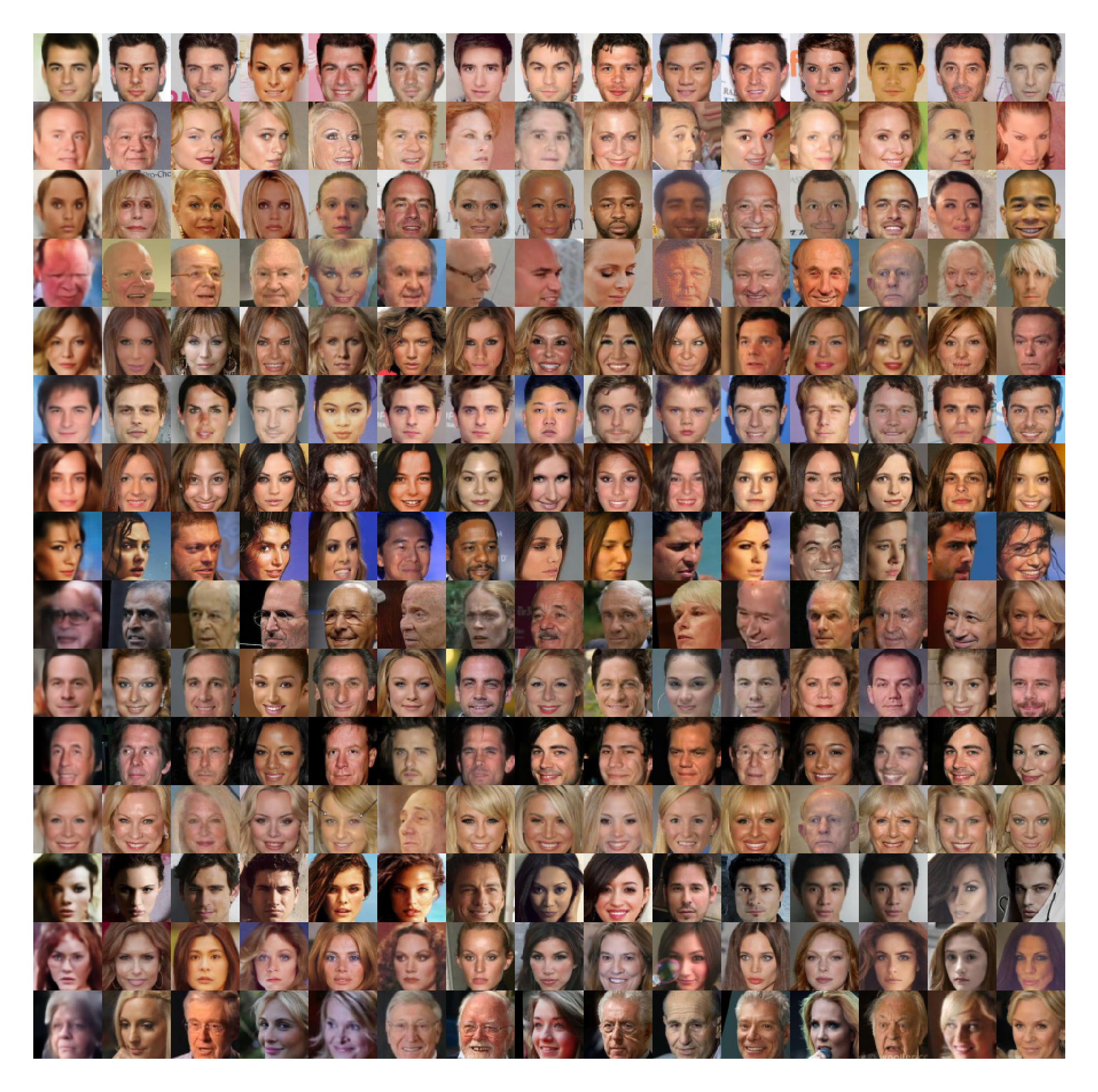}
    \caption{The first entry in each row represents a randomly generated face using FlexAE. The remaining entries in each row represents $14$ nearest neighbours (in terms of Euclidean distance) from the train split of CELEBA dataset. It is seen that the generated images using FlexAE are very different as compared to the training examples. This confirms that the state of the art FID score and precision recall score obtained using FlexAE is not due to mere overfitting on the training split.}
    \label{fig:supp_knn_15x15}
\end{figure*}

\clearpage
\section{\color{white}{References}}
\vspace{-93mm}
\small{
\setlength{\bibsep}{5pt}
\bibliography{bibliography}}

\end{document}